\documentclass[preprint,review,3p,11pt]{elsarticle} 

\usepackage[caption=false,font=normalsize,labelfont=sf,textfont=sf]{subfig}
\usepackage{textcomp}

\usepackage{amsmath,amsfonts,amssymb}
\usepackage{algorithmic}
\usepackage{algorithm}
\usepackage{array}
\usepackage{textcomp}
\usepackage{stfloats}
\usepackage{url}
\usepackage{verbatim}
\usepackage{graphicx}
\usepackage{hyperref}
\usepackage{algorithm}
\usepackage{algorithmic}
\usepackage{extarrows}
\usepackage{physics}
\usepackage{bm}
\usepackage{bbm}
\hypersetup{
	colorlinks=true,
	linkcolor=red,
	filecolor=red,
	urlcolor=red,
	citecolor=green,
}
\usepackage{amsthm}

\usepackage[table]{xcolor}
\newtheorem{thm}{\bf Theorem}

\usepackage{graphicx}
\usepackage{float}
\usepackage{enumitem}

\DeclareSubrefFormat{myparens}{#1(#2)}
\usepackage{comment}
\usepackage{setspace}

\usepackage{threeparttable}
\usepackage{booktabs}
\usepackage{makecell}
\usepackage{multirow}

\journal{arXiv}

\begin{document}

\begin{frontmatter}



\title{Bi-level Unbalanced Optimal Transport for Partial Domain Adaptation}

\author[label1]{Zi-Ying~Chen}
\author[label1]{Chuan-Xian~Ren\corref{cor1}}
\ead{rchuanx@mail.sysu.edu.cn}
\author[label2]{Hong~Yan}

\affiliation[label1]{organization={School of Mathematics, Sun Yat-Sen University},
             city={Guangzhou},
             postcode={510275},
            state={Guangdong},
            country={China}}
\affiliation[label2]{organization={Department of Electrical and Engineering, City University of Hong Kong},
             addressline={83 Tat Chee Avenue},
             city={Kowloon},
             postcode={999077},
            state={Hong Kong}, 
            country={China}}

\cortext[cor1]{Corresponding author}

\begin{abstract}
Partial domain adaptation (PDA) problem requires aligning cross-domain samples while distinguishing the outlier classes for accurate knowledge transfer. The widely used weighting framework tries to address the outlier classes by introducing the reweighed source domain with a similar label distribution to the target domain. However, the empirical modeling of weights can only characterize the sample-wise relations, which leads to insufficient exploration of cluster structures, and the weights could be sensitive to the inaccurate prediction and cause confusion on the outlier classes. To tackle these issues, we propose a Bi-level Unbalanced Optimal Transport (BUOT) model to simultaneously characterize the sample-wise and class-wise relations in a unified transport framework. Specifically, a cooperation mechanism between sample-level and class-level transport is introduced, where the sample-level transport provides essential structure information for the class-level knowledge transfer, while the class-level transport supplies discriminative information for the outlier identification. The bi-level transport plan provides guidance for the alignment process. By incorporating the label-aware transport cost, the local transport structure is ensured and a fast computation formulation is derived to improve the efficiency. Extensive experiments on benchmark datasets validate the competitiveness of BUOT.
\end{abstract}

\begin{keyword}
Partial domain adaptation\sep Unbalanced optimal transport\sep Class weight\sep Optimal transport plan\sep Distribution discrepancy
\end{keyword}
\end{frontmatter}

\section{Introduction}
Traditional machine learning usually follows the assumption that training data and test data come from the same distribution. However, in real-world scenarios, data collected from different devices, environments, or at different times may exhibit distribution shifts, leading to distribution discrepancy between the datasets. This distribution discrepancy can degrade the performance of machine learning models when they are deployed in new environments or domains. To overcome this challenge, unsupervised domain adaptation (UDA)~\cite{kerdoncuff2021metric,XU2023109787} has been developed to transfer knowledge from the labeled source domain to the unlabeled target domain, enabling the models trained on the source domain that can generalize well to the target domain.

Usually, UDA methods train the model using source domain samples to minimize the source domain classification error and then use appropriate methods to eliminate the cross-domain divergence. Methods to eliminate the divergence include minimizing statistical distance~\cite{tzeng2014deep, kang2019contrastive, xia2023maximum} and domain adversarial learning~\cite{ganin2016domain, dhouib2020margin}. Commonly used statistical distances include maximum mean discrepancy (MMD)~\cite{long2018transferable,thota2021contrastive} and Wasserstein distance~\cite{zhang2020optimal,chen2018re,shen2018wasserstein}. In the UDA problem, it is typically assumed that the source and target domains share the same label space. However, this assumption may not always hold in real-world applications, making it difficult to find a source domain that has the same label space as the target domain. Finding a new source domain to assist the target domain learning is very complicated. Therefore, exploring alternative methods becomes necessary, and an effective approach is to leverage existing large-scale labeled datasets. It is desirable to transfer models trained on the large-scale labeled datasets (e.g., ImageNet~\cite{russakovsky2015imagenet}) to smaller datasets (e.g., Caltech-256~\cite{griffin2007caltech}) and enhance the performance of the models on the smaller datasets. Since the label space of large-scale datasets often does not exactly match the label space of small datasets, vanilla UDA methods are not directly applicable. Partial domain adaptation (PDA)~\cite{cao2018partial, ren2020learning, yang2023contrastive} is developed to address this issue.

\begin{figure}[!t]
    \centering
    \includegraphics[width=0.6\linewidth,trim=22 32 22 30,clip]{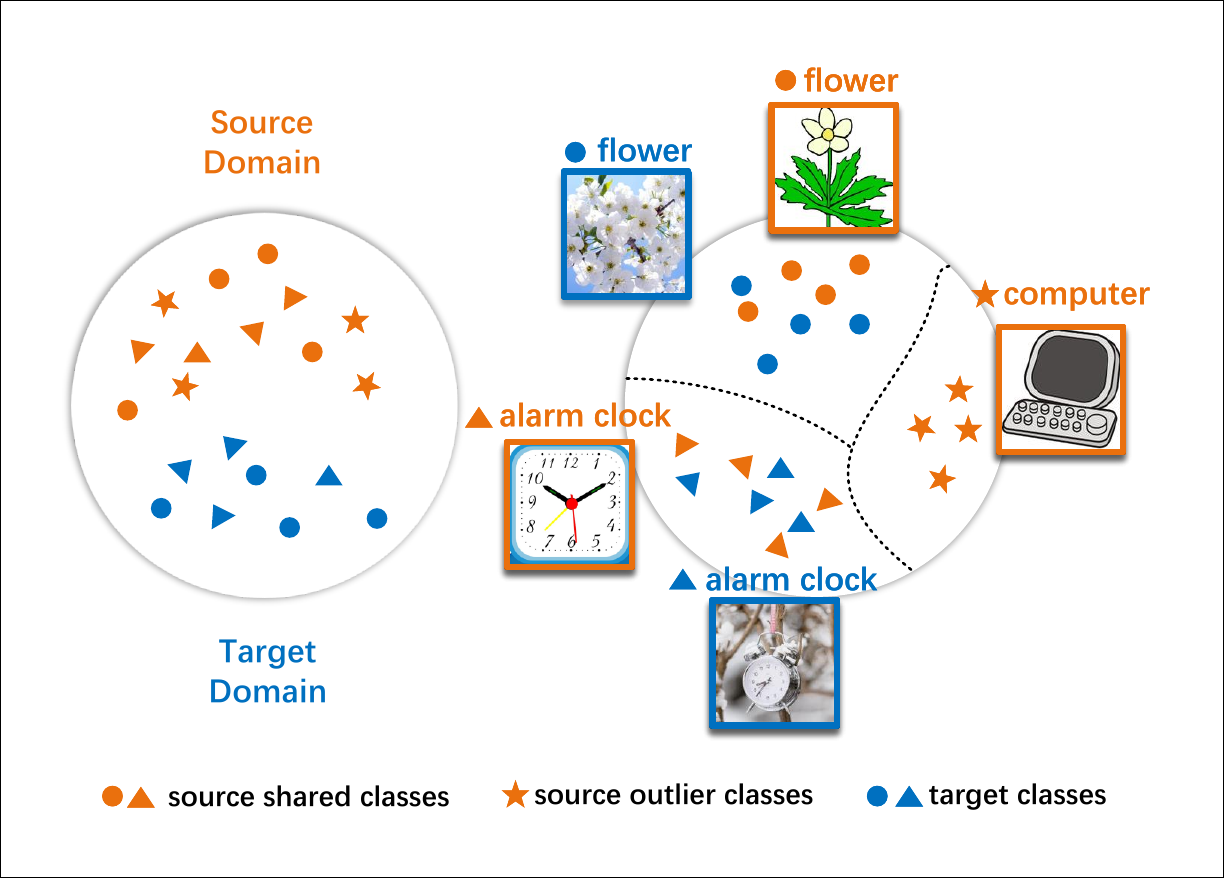}
    \caption{Illustration of partial domain adaptation (PDA) problem. PDA assumes that the label space of the target domain is a subset of the source domain. The goal of PDA is to achieve cross-domain alignment while distinguishing between shared classes and outlier classes in the source domain. Best viewed in color.}\label{fig:PDA} 
\end{figure}

PDA assumes that the label space of the target domain is a subspace of the label space of the source domain. As shown in Fig.~\ref{fig:PDA}, the class 'computer' in the source domain does not exist in the target domain. Classes in the source domain that overlap with the target domain are called shared classes, while classes that exist only in the source domain are called outlier classes. In this case, directly matching the entire source domain to the target domain can lead to negative transfer, as outlier class samples may have an adverse effect on the model. Therefore, the PDA problem not only needs to align the source and target domains, but also to identify outlier classes to mitigate the impact of negative transfer. By aligning the shared classes in both domains, PDA methods can improve the generalization ability of the model.

Mainstream PDA methods mitigate the adverse effects of outlier classes by weighting the source domain. Since the outlier classes do not exist in the target domain, the probability of predicting target samples as outlier classes is relatively small. Thus, the mean value of all the target predictions can be used as class-level weights. Elements of the class-level weights measure the probability that source classes belong to the source shared class, and larger elements suggest a higher probability of belonging to the shared classes. Methods such as partial adversarial domain adaptation (PADA)~\cite{cao2018partialada}, selective adversarial networks (SAN)~\cite{cao2018partial} and discriminative manifold propagation (DMP) employ the class-level weights. If there exists a class imbalance case in the target domain, using class-level weights may result in poor classification performance for target classes with fewer samples, prompting some methods to employ sample-level weights to weigh the source domain. The improved version of SAN (SAN++)~\cite{cao2023big} uses the prediction probabilities of samples to represent the probabilities of those samples belonging to different source classes. Cao et al.~\cite{cao2019learning} propose example transfer network (ETN), which uses the output of an auxiliary domain discriminator to measure the transferability of source samples and assigns weights based on this transferability. However, whether using class-level weights or sample-level weights, the above methods derive weights based on the relations between samples, lacking full exploration of cluster structures. If the predictions are inaccurate, the weighted class distributions of the source domain are still far from the class distributions of the target domain. 

Due to the solid mathematical foundation and effective application results of optimal transport (OT)~\cite{luo2021conditional, Wang2024pp, qian2022joint}, many OT-based methods have been widely used for domain adaptation, such as joint distribution optimal transport (JDOT)~\cite{courty2017joint} and enhanced transport distance (ETD)~\cite{li2020enhanced}. In the domain adaptation problem, the source and target domains are considered as two different distributions, and the Wasserstein distance can measure the divergence between the two distributions. Although OT-based methods have been widely applied to domain adaptation problem, traditional OT methods require both measures to have the same total probability mass. This makes the traditional OT methods prone to errors in the PDA problem, where class distributions in the source and target domains are inconsistent. Using traditional OT models in the PDA problem will learn incorrect sample relations, thus performing incorrect inter-class sample transport. Unbalanced optimal transport (UOT)~\cite{chizat2018unbalanced, chizat2018scaling} is a relaxed version of traditional OT that replaces the marginal constraints with penalty terms, allowing for changes in probability mass during the transport process. More specifically, UOT relaxes the strict constraints on the transport plan by allowing the outlier class samples to transport only a small amount of probability mass while the shared class samples transport more. This makes UOT more suitable for the PDA problem than traditional OT methods.

In this paper, we propose a Bi-level Unbalanced Optimal Transport (BUOT) model to address the PDA problem. Unlike existing weighting-based methods that rely solely on sample-wise relations, BUOT integrates sample-wise and class-wise information to derive bi-level weights. Specifically, BUOT simultaneously learns the transport plan between samples and classes in the source and target domains. The sample-level transport plan captures fine-grained relations to guide class-level alignment, while the class-level transport plan provides the discriminative information necessary for sample-level transport to identify outlier classes. We recover the sample-wise and class-wise relations through the learned bi-level transport plan and subsequently leverage the transport relations to obtain bi-level weights. These weights are then applied to reweight the source domain distribution, ensuring that the corrected source class distribution is similar to the target class distribution. To address the class distribution mismatch between the source and target domains, we utilize UOT to learn transport plan. UOT relaxes the strict mass conservation constraint, allowing for changes in transport mass during transportation. To ensure accurate classification, we propose a novel label-aware transport cost that reduces costs within the same class while increases costs between different classes. Additionally, we derive a fast computation formulation to enhance computational efficiency. In summary, the contributions of our work can be summarized as follows.

\begin{itemize}[noitemsep]
\item To reduce the impact of inaccurate target predictions, a new PDA model is proposed to simultaneously learn the sample-level and class-level transport plans between the source and target domains. The sample-level transport plan and class-level transport plan are integrated and learned together.
 \item The learned bi-level optimal transport plan can integrate class-wise and sample-wise information to recover the explicit expression of sample-wise and class-wise relations. Bi-level weights can be obtained, thus distinguishing between shared and outlier classes in the source domain.
 \item To learn more discriminative representations, a novel label-aware transport cost applied to BUOT is proposed. To improve computational efficiency, a fast computation formulation via matrix-vector multiplication for this cost has been derived, ensuring faster computations within the BUOT model.
\end{itemize}

The rest of this paper is organized as follows. Section~\ref{chap:2} presents the related works of this paper. The methodology and algorithm for BUOT in PDA are presented in Section~\ref{chap:3}. In Section~\ref{chap:4}, the effectiveness of BUOT is validated through experiments. Finally, Section~\ref{chap:5} provides conclusions.

\section{Related Work}  \label{chap:2}

\subsection{Partial Domain Adaptation}
Partial domain adaptation assumes that the label space of the target domain is a subspace of the source domain, so it is not only necessary to improve the generalizability of the model by aligning the shared classes between the source and target domains, but also needs to identify the outlier classes to mitigate the effect of negative transfer.

Many PDA methods~\cite{cao2018partial, zhang2018importance, luo2022unsupervised} increase the weights of the shared class samples in the source domain and reduce the weights of the outlier class samples. The weighting-based methods aim to make the label space of the reweighted source domain closer to the target domain. Early works mainly focus on adding weights to adversarial networks. For example, Cao et al. propose SAN~\cite{cao2018partial} and PADA~\cite{cao2018partialada}, both computing probability-weighted adversarial loss or classification loss by prediction probability to reduce the impact of outlier class samples. Zhang et al.~\cite{zhang2018importance} propose importance weighted adversarial nets (IWAN), which use an auxiliary domain discriminator to calculate the probabilities of source samples belonging to shared classes or outlier classes.

In addition to these classical weighting methods, different weighting strategies were subsequently proposed. Unlike methods based on adversarial learning, Li et al.~\cite{li2020deep} propose deep residual correction network (DRCN) to reduce cross-domain divergence using MMD. Compared to previous weighting methods that only consider the predictions of the source domain classifiers, Yang et al.~\cite{yang2023contrastive} propose a weighting scheme considering the weights generated by the target domain information. Lin et al.~\cite{lin2022adversarial} propose that outlier class samples are more likely to change classes after cycle transformation compared to shared class samples in the source domain, and therefore cycle inconsistency can be used to filter out the outlier class samples. Gu et al.~\cite{gu2021adversarial} learn weights by minimizing the Wasserstein distance between the distributions of the reweighted source domain and the target domain.

Besides weighting-based methods, many other PDA methods have been proposed in recent years. Liang et al.~\cite{liang2020balanced} believe that using target predictions to weight the source domain depends on the accuracy of target predictions. Therefore, the authors propose to use source samples to augment the target domain instead of weighting the source domain. Methods based on reinforcement learning~\cite{wu2023reinforced, chen2022domain} no longer weigh the source domain but directly select shared class samples from the source domain. For example, Chen et al.~\cite{chen2022domain} propose the deep reinforcement learning-based source data selector to determine whether to retain or discard source samples, thereby achieving more precise knowledge transfer.

Although our method also utilizes weights, we integrate both sample-wise and class-wise information, rather than deducing weights solely from the relations between samples. Weighting methods dependent on target predictions are prone to error when target prediction accuracy is compromised. We consider essential relations through sample-wise information and inherent discriminative properties through class-wise information, thereby reducing the extent to which the class weights deviate from reality.

\subsection{Optimal Transport}

Optimal transport was first proposed by Monge~\cite{peyre2017computational}. It aims to find a way to move a pile of sand with a certain shape into a specified pit of another shape with minimal transport cost. Since the Monge problem is difficult to solve, Kantorovich~\cite{kantorovich1942translocation} relaxes the conditions to solve the optimal transport plan by optimizing the coupling matrix. Suppose $\mathcal{X}_1$ and $\mathcal{X}_2$ are complete metric spaces with probability measures $\mu$ and $\nu$, respectively. $\Pi(\mu, \nu)$ represents the set of probability couplings between $\mu$ and $\nu$. Let the cost function be $c\colon\mathcal{X}_1\times\mathcal{X}_2\rightarrow[0,+\infty]$. In the subsequent discussion, $c$ and $\gamma$ specifically refer to 
$c(\bm{x}_1,\bm{x}_2)$ and $\gamma(\bm{x}_1,\bm{x}_2)$, respectively, with $\bm{x}_1\in\mathcal{X}_1,\bm{x}_2\in\mathcal{X}_2$. The mathematical definition of Kantorovich problem is formulated as
\begin{equation}
K(\mu,\nu)= \mathop{\min}_{\gamma\in\Pi(\mu, \nu)}\int_{\mathcal{X}_1\times\mathcal{X}_2} c \dd{\gamma}.
\end{equation}

Considering the high computational complexity and excessively sparse solutions of the Kantorovich problem, the Sinkhorn distance \cite{cuturi2013sinkhorn} provides an approximate solution by adding an entropy regularization term. The entropy of $\gamma$ is defined as $H(\gamma)=\mathbb{E}_\gamma[-\ln(\dd{\gamma})]$. Then the entropy-regularized OT problem can be expressed as follows
\begin{equation}
K(\mu,\nu) = \mathop{\min}_{\gamma\in\Pi(\mu, \nu)}\int_{\mathcal{X}_1\times\mathcal{X}_2} c\dd{\gamma}-\lambda H(\gamma),
\end{equation}
where $\lambda$ is penalty parameter.

The traditional OT problem provides mappings that preserve total mass. However, in real-world applications, encountering balanced data is uncommon. To address scenarios with unequal transport and reception mass, Unbalanced Optimal Transport (UOT)~\cite{chizat2018unbalanced} introduce a relaxed penalty term on the transport coupling rather than imposing strict marginal constraints $\gamma \in \Pi(\mu, \nu)$. Suppose $\varphi$-divergence is defined as $D_{\varphi}(\mu\|\nu)=\mathbb{E}_{\nu}[\phi(\frac{d\mu}{d\nu})]$, the formula for UOT as follows

\begin{equation}
U(\mu,\nu)=\mathop{\min}_{\gamma\in\mathcal{M}_{+}}\int_{\mathcal{X}_1\times\mathcal{X}_2} c d{\gamma} -\lambda H(\gamma)
          +\beta[D_{\varphi}(\gamma_{\mu}||\mu)+D_{\varphi}(\gamma_{\nu}||\nu)],
\end{equation}
where $\gamma_{\mu}$ and $\gamma_{\nu}$ are the margins of $\gamma$, $\beta$ is the parameter of the marginal penalty, $\mathcal{M}_{+}$ is the distribution space and $D_{\varphi}(P\|Q)=\mathbb{E}_{Q}[\phi(\frac{dP}{dQ})]$ is $\varphi$-divergence. When $\beta\rightarrow +\infty$, UOT problem degenerates into traditional OT problem. For general $\beta$, UOT relaxes the strict constraints on $\gamma$, allowing outlier points to transport or receive a smaller probability mass while key points can transport or receive a larger probability mass.

Furthermore, in traditional OT methods, two datasets are usually required to be in the same dimensional space to calculate the transport cost between samples. This limits the application of traditional OT methods when dealing with heterogeneous datasets. To solve this problem, one can use the Gromov-Wasserstein (GW) distance~\cite{memoli2011gromov} to avoid calculating the cost between sample pairs in spaces with different dimensions. The GW distance aims to calculate the difference between similarities of sample pairs. Therefore, Peyre et al.~\cite{peyre2016gromov} introduce it into the OT problem to calculate the matching relation between intra-domain sample pairs similarity in different dimensional domains. Furthermore, Titouan et al.~\cite{titouan2020co} propose CO-Optimal Transport (COOT), which considers the transport mapping between samples and between features of two datasets with any dimensions. GW is a special case of COOT, and COOT can directly calculate the transport between original data without calculating the similarity between samples. 

\section{Method}  \label{chap:3}
In this paper, let $\mathcal{X}$ be the space of continuous inputs, and $\mathcal{Y}$ be the space of discrete labels. In the PDA setting, the target label space $\mathcal{Y}_{t}$ is a subspace of the source label space $\mathcal{Y}_{s}$, that is, $\mathcal{Y}_{t}\subset\mathcal{Y}_{s}$. Assume $|\mathcal{Y}_{s}|=K$, for empirical scenarios, the finite samples from the labeled source domain and unlabeled target domain are denoted as $\mathcal{D}_s=\{(\bm{x}_{i_1}^s,y_{i_1}^s)\}_{i_1=1}^{n_s}$ and $\mathcal{D}_t=\{\bm{x}_{j_1}^t\}_{j_1=1}^{n_t}$, where $n_s$ and $n_t$ are the sample-sizes of source and target domains, respectively. The basic model consists of a representation learner $f:\mathcal{X}\to \mathcal{Z}$ and task predictor $h:\mathcal{Z}\to \mathcal{Y}$, where $\mathcal{Z}$ is the latent representation space. The probabilistic prediction defined as $\bm{p}=h\circ f (\bm{x})\in \mathbb{R}^K$.

\subsection{Bi-level Unbalanced Optimal Transport}

To effectively address the PDA problem, it requires align the source and target domains and identify outlier classes. Common PDA methods only consider sample-wise relations and assign weights based on target predictions. However, such prediction-based weights can be impacted by errors in predictions, which in turn can hinder the learning of relations between classes. We aim to learn both sample-level and class-level transport plans simultaneously, where the sample-level and class-level transport plans interact and promote each other. Specifically, sample-level transport plan provides essential relations to facilitate class-level transport learning, while class-level transport guides the model to perform correct sample-level intra-class transport through discriminative information. By considering the bi-level optimal transport simultaneously, we can reduce the impact of incorrect predictions on the inference of class relations. Thus, we propose a bi-level unbalanced optimal transport (BUOT) model.

\begin{figure}[t]
    \centering
    \includegraphics[width=0.6\linewidth,trim=45 34 45 42,clip]{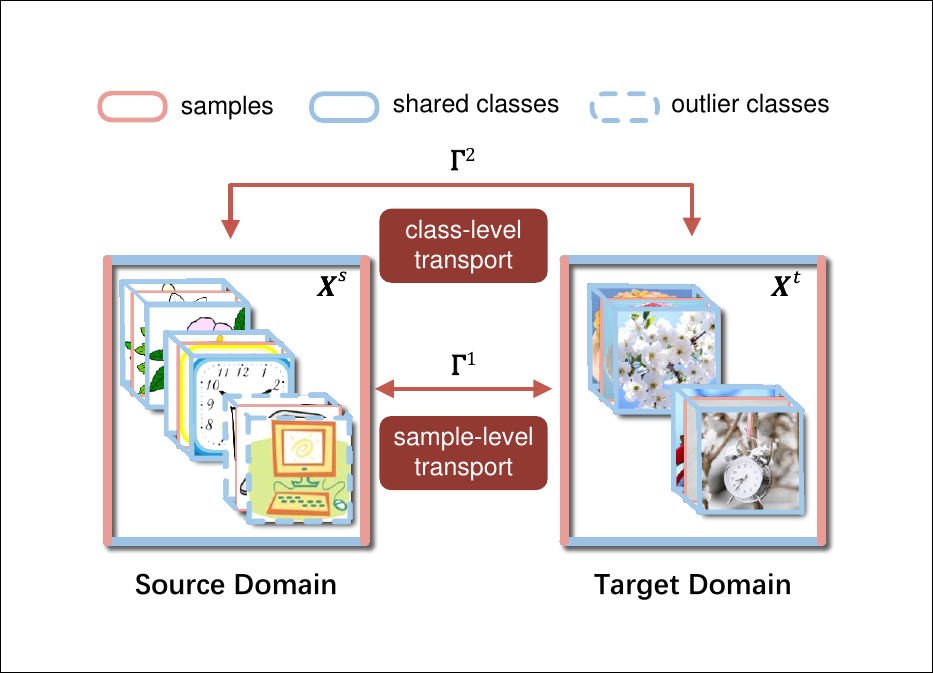}
    \caption{Illustration of BUOT. Sample-level transport plan $\bm{\Gamma}^1$ provides structure information for the class-level knowledge transfer, while class-level transport plan $\bm{\Gamma}^2$ provides discriminative information for the outlier class identification. Best viewed in color.}
    \label{fig:BUOT}
\end{figure}

The illustration of the BUOT model is shown in Fig.~\ref{fig:BUOT}, where $\bm{\Gamma}^1$ is the sample-level transport plan and $\bm{\Gamma}^2$ is the class-level transport plan. $\bm{\Gamma}^1$ provides structure information for the class-level knowledge transfer, while $\bm{\Gamma}^2$ provides discriminative information for the outlier class identification. Solid lines represent shared classes and dashed lines represent outlier classes. By jointly optimizing $\bm{\Gamma}^1$ and $\bm{\Gamma}^2$, BUOT aims to transport between the shared classes in both domains.

The set of prediction vectors for all samples in the source domain is $\{\bm{p}_i^s\}_{i=1}^{n_s}$, and the matrix composed of all the prediction vectors is $\bm{P}^s$. Similarly, the matrix composed of all predicted vectors in the target domain is $\bm{P}^t$. We can leverage COOT to achieve the bi-level transport. In the sample-level perspective, the empirical distributions are denoted as $\mu_1=\frac{1}{n_s}\sum_{i_1=1}^{n_s}[\bm{u}_{1}]_{i_1}\delta_{\bm{p}_{i_1}^s}$ and $\nu_1=\frac{1}{n_t}\sum_{j_1=1}^{n_t}[\bm{v}_{1}]_{j_1}\delta_{\bm{p}_{j_1}^t}$, respectively, where $\sum_{i_1=1}^{n_s}[\bm{u}_{1}]_{i_1}=1$ and $\sum_{j_1=1}^{n_t}[\bm{v}_{1}]_{j_1}=1$, $\delta_{\bm{p}}$ is the Dirac function at position $\bm{p}$. In the class-level perspective, the empirical distributions are denoted as $\mu_2=\frac{1}{K}\sum_{i_2=1}^{K}[\bm{u}_{2}]_{i_2}\delta_{\bm{p}_{i_2}^s}$ and $\nu_2=\frac{1}{K}\sum_{j_2=1}^{K}[\bm{v}_{2}]_{j_2}\delta_{\bm{p}_{j_2}^t}$, where $\sum_{i_2=1}^{K}[\bm{u}_{2}]_{i_2}=1$ and $\sum_{j_2=1}^{K}[\bm{v}_{2}]_{j_2}=1$. $\Pi(\cdot,\cdot)$ is the set of probabilistic couplings, i.e., $\Pi(\bm{u}_1,\bm{v}_1)=\{\bm{\Gamma}^1\in\mathbb{R}_+^{n_s\times n_t}|\bm{\Gamma}^1\bm{1}_{n_t}=\bm{u}_1,\bm{\Gamma}^{1T}\bm{1}_{n_s}=\bm{v}_1\}$. Then the bi-level optimal transport model is built as
\begin{equation}
\mbox{BOT}(\bm{P}^s,\bm{P}^t) = \mathop{\min}_{\substack{\bm{\Gamma}^1\in\Pi(\bm{u}_1,\bm{v}_1)\\\bm{\Gamma}^2\in\Pi(\bm{u}_2,\bm{v}_2)}} \sum_{i_1=1}^{n_s}\sum_{j_1=1}^{n_t}\sum_{i_2=1}^{K}\sum_{j_2=1}^{K}C(P_{i_1i_2}^s, P_{j_1j_2}^t)\Gamma_{i_1j_1}^1\Gamma_{i_2j_2}^2,
\end{equation}
where $C(\cdot,\cdot)$ is the cost function between two one-dimensional variables, the squared Euclidean distance is usually used as the cost function, i.e., $C(\cdot,\cdot)= \| \cdot - \cdot \|_2^2$. $P_{i_1i_2}^{s}$ represents the probability that the $i_1$-th source sample being predicted as the $i_2$-th class.

For any tensor $\bm{L}=(L_{ijkl})$ and matrices $\bm{A}$, $\bm{B}$, we use $\langle \cdot, \cdot \rangle_F$ to denote the Frobenius inner product, defined as $\langle \bm{A}, \bm{B} \rangle_F = \sum_{i,j} A_{ij} B_{ij}$, and $\otimes$ to represent the tensor-matrix multiplication, defined as follows
\begin{equation}
\label{eq:tensor-matrix}
\bm{L}\otimes\bm{A}\overset{\mathrm{def.}}{=}(\sum_{kl}L_{ijkl}A_{kl})_{ij}.
\end{equation}

Then the bi-level optimal transport can rewritten as
\begin{equation}
\mbox{BOT}(\bm{P}^s,\bm{P}^t)= \mathop{\min}_{\substack{\bm{\Gamma}^1\in\Pi(\bm{u}_1,\bm{v}_1)\\\bm{\Gamma}^2\in\Pi(\bm{u}_2,\bm{v}_2)}}\langle \bm{C}(\bm{P}^s, \bm{P}^t)\otimes\bm{\Gamma}^1, \bm{\Gamma}^2 \rangle_F\\
 = \mathop{\min}_{\substack{\bm{\Gamma}^1\in\Pi(\bm{u}_1,\bm{v}_1)\\\bm{\Gamma}^2\in\Pi(\bm{u}_2,\bm{v}_2)}}\langle \mathbf{C}(\bm{P}^s, \bm{P}^t)\otimes\bm{\Gamma}^2, \bm{\Gamma}^1 \rangle_F,
\end{equation}
where $\mathbf{C}(\bm{P}^s, \bm{P}^t)$ is a 4-order tensor.

Considering the computational complexity and preventing the sparsity of the solution, we introduce the entropy regularizer $H(\bm{\Gamma}^1)=\mathbb{E}_{\bm{\Gamma}^1}[-\ln(\dd{\bm{\Gamma}^1})]$ and $H(\bm{\Gamma}^2)=\mathbb{E}_{\bm{\Gamma}^2}[-\ln(\dd{\bm{\Gamma}^2})]$ for $\bm{\Gamma}^1$ and $\bm{\Gamma}^2$. For simplicity, denote $H(\bm{\Gamma}^1, \bm{\Gamma}^2) = \lambda_1 H(\bm{\Gamma}^1) + \lambda_2 H(\bm{\Gamma}^2)$. Then the entropic regularized bi-level optimal transport can be written as
\begin{equation}
\mbox{BOT}(\bm{P}^s,\bm{P}^t)= \mathop{\min}_{\substack{\bm{\Gamma}^1\in\Pi(\bm{u}_1,\bm{v}_1)\\\bm{\Gamma}^2\in\Pi(\bm{u}_2,\bm{v}_2)}}\langle \mathbf{C}(\bm{P}^s, \bm{P}^t)\otimes\bm{\Gamma}^1, \bm{\Gamma}^2 \rangle_F - H(\bm{\Gamma}^1, \bm{\Gamma}^2).
\end{equation}

We aim to learn discriminative transport plans, which enable correct cross-domain intra-class transport. However, fewer target classes than source classes in PDA problems cause a difference between the mass of the source domain transport and the target domain received. Traditional OT imposes strict constraints on the transport plan, requiring that the mass remains unchanged during transport, hence it is unable to solve the class number inconsistency issue in PDA. 
UOT relaxes the constraint, allowing outlier points only to transport or receive a small amount of probability mass, while critical points can transport or receive more probability mass. This is in line with the needs of the PDA problem. Using UOT to solve the PDA problem can enable shared class samples to transport more probability mass, while outlier class samples transport little or no probability mass. Therefore, we apply UOT to the bi-level optimal transport model. 
The relaxed penalty with respect to $\bm{\Gamma}^1$ and $\bm{\Gamma}^2$ are $D(\bm{\Gamma}^1)= D_{\varphi}(\bm{\Gamma}^1_{\mu_1}||\mu_1)+D_{\varphi}(\bm{\Gamma}^1_{\nu_1}||\nu_1)$ and $D(\bm{\Gamma}^2)= D_{\varphi}(\bm{\Gamma}^2_{\mu_2}||\mu_2)+D_{\varphi}(\bm{\Gamma}^2_{\nu_2}||\nu_2)$, respectively. For simplicity, denoting that $D(\bm{\Gamma}^1, \bm{\Gamma}^2)  =  \beta_1 D(\bm{\Gamma}^1) + \beta_2 D(\bm{\Gamma}^2)$. The BUOT strategy is shown as follows 
\begin{equation}
\mbox{BUOT}(\bm{P}^s,\bm{P}^t) = \mathop{\min}_{\substack{\bm{\Gamma}^1\in\mathcal{M}_{+}(\mathcal{X}^2)\\\bm{\Gamma}^2\in\mathcal{M}_{+}(\mathcal{Y}^2)}}\langle \mathbf{C}(\bm{P}^s, \bm{P}^t)\otimes\bm{\Gamma}^1, \bm{\Gamma}^2 \rangle_F- H(\bm{\Gamma}^1, \bm{\Gamma}^2) + D(\bm{\Gamma}^1, \bm{\Gamma}^2),
\end{equation}
where $\mathcal{M}_{+}$ denote the distribution space, with $\mathcal{M}_{+}(\mathcal{X}^2)$ and $\mathcal{M}_{+}(\mathcal{Y}^2)$ defined over the spaces $\mathcal{X}^2$ and $\mathcal{Y}^2$, respectively.

\begin{figure*}[t]
    \centering
    \includegraphics[width=0.9\linewidth,trim=46 31 57 29,clip]{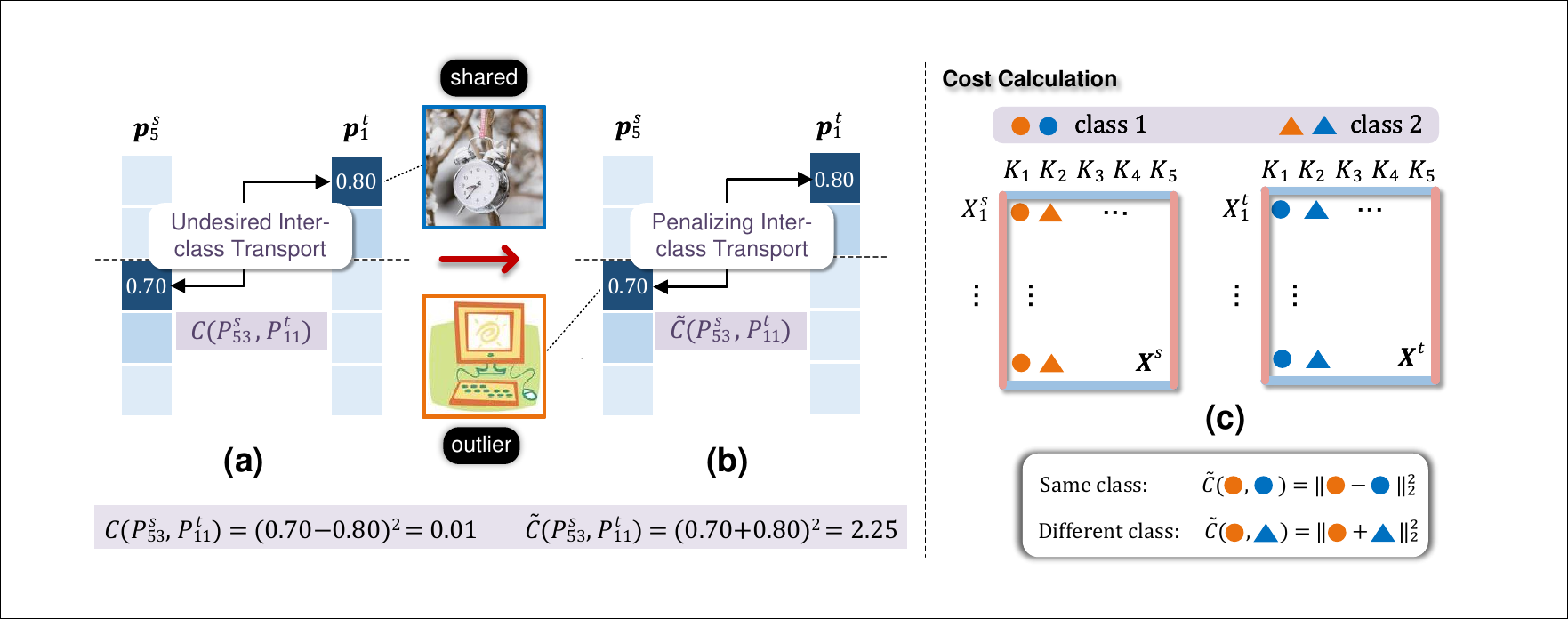}
    \caption{Illustration of the label-ware transport cost. $\bm{p}_5^s$ and $\bm{p}_1^t$ represent the prediction vectors for the source outlier sample $\bm{x}_5^s$ and the target sample $\bm{x}_1^t$, respectively. The dashed line separates the shared classes and the outlier classes. The squared Euclidean distance is denoted as $C(\cdot,\cdot)$, while the label-ware transport cost is represented as $\tilde{C}(\cdot,\cdot)$. Different colors indicate different classes and different shapes indicate different domains. Best viewed in color.}
    \label{fig:Cost}
\end{figure*}

\subsection{Label-aware Transport Cost }
To better refine the recognition and alignment of classes in the BUOT model, we propose a novel label-aware transport cost. As shown in Fig.~\ref{fig:Cost}, the left diagram illustrates transport cost computation between the cross-domain prediction vectors. There are ten samples in each domains, denoted as $(x_{i_1}^s), i_1 \in [10]$ and $(x_{j_1}^t), j_1 \in [10]$, respectively. The source domain has five classes, where $K_1,K_2$ are shared classes and the remaining classes are outlier classes. Specifically, source samples $\{x_1^s,x_2^s\},\{x_3^s,x_4^s\},\{x_5^s,x_6^s\}$,\\$\{x_7^s,x_8^s\},\{x_9^s,x_{10}^s\}$ belong to classes $K_1$, $ K_2$, $ K_3$, $ K_4$ and $ K_5$ respectively. In the target domain, $x_1^t$ to $x_5^t$ belong to class $K_1$, while $x_6^t$ to $x_{10}^t$ belong to class $K_2$. 

If $P_{53}^s = 0.7$ and $P_{11}^t = 0.8$, it means that the outlier sample $x_5^s$ has a probability of 0.7 to be predicted as the class $ K_3$ while target sample $x_1^t$ has a probability of 0.8 to be predicted as the class $ K_1$. In Fig.~\ref{fig:Cost}(a), using squared Euclidean distance as the cost, the value of $C(P_{53}^s, P_{11}^t)$ is 0.01, indicating a high probability of matching the target sample to the outlier class in the source domain. This mismatch indicates that the squared Euclidean distance used in BUOT fails to properly address the challenge of recognizing the outlier classes in the PDA problem, leading to undesired inter-class transport. In addition, when utilizing BUOT to align the source and target domains, it is crucial to incorporate more discriminative information to avoid incorrect alignment and transport. This means achieving small transport costs for intra-class samples and large transport costs for inter-class samples. To address these two issues, we propose a label-aware transport cost, denoted as $\tilde{C}(\cdot,\cdot)$. It is designed to better ensure the local transport structure, and can also address the problem of mismatching the target domain to outlier classes in the source domain.

The label-aware transport cost increase the inter-class sample cost according to the label indices, thereby penalizing inter-class transport. As shown in Fig.~\ref{fig:Cost}(c), the index $i_2$ of $P_{i_1i_2}^s$ in $C(P_{i_1i_2}^s, P_{j_1j_2}^t)$ represents the probability of being predicted as the $i_2$-th class. If the label indices are the same, i.e., $i_2=j_2$, the cost remains unchanged. However, if the label indices differ, we consider replacing the subtraction with addition, such as $| \cdot + \cdot |_2^2$, to increase the cost. Thus, the formulation of the new cost is given as follows

\begin{align}
\label{eq:cost_mechanism}
\begin{split}
\tilde{C}(P_{i_1i_2}^s, P_{j_1j_2}^t)= \left \{
\begin{array}{ll}
    \| P_{i_1i_2}^s - P_{j_1j_2}^t \|_2^2,        & i_2=j_2,\\
    \| P_{i_1i_2}^s + P_{j_1j_2}^t \|_2^2,        & i_2\neq j_2.
\end{array}
\right.
\end{split}
\end{align}

This new label-aware cost ensures that the cross-domain inter-class transport cost is larger than the cross-domain intra-class transport cost, thus avoiding matching the target samples to the outlier classes in the source domain to some extent. In Fig.~\ref{fig:Cost}(b), it can be seen that using the label-aware transport cost, the calculated cross-domain transport cost $\tilde{C}(P_{53}^s, P_{ 11}^t) = 2.25$, which can avoid assigning target samples to the outlier classes. As the prediction probabilities are bounded within the range of $[0, 1]$, if the source sample $\bm{x}_{i_1}^s$ and the target sample $\bm{x}_{j_1}^t$ belong to the same class, the cross-domain transport cost will be less than $\| 1 - 0 \|_2^2=1$. So $\bm{x}_{1}^t$ does not belong to the same class as $\bm{x}_{5}^s$, which reduces the risk of being recognized as the outlier class sample. Therefore, the new cost enables the BUOT model to better address these two issues, improving not only the discriminability of the model but also the identification of outlier classes.

In general, the computation of the tensor-matrix multiplication is complicated. Peyre et al.~\cite{peyre2016gromov} propose that if the loss function satisfies a specific decomposition, the complex tensor-matrix multiplication calculation can be simplified to matrix calculation. If $f_1,f_2,h_1,h_2$ denote functions, the decomposition is as follows
\begin{equation}
\label{eq:decomposition}
L(a,b)=f_1(a)+f_2(b)-h_1(a)h_2(b).
\end{equation}

Although many functions satisfy this decomposition, such as squared Euclidean distance, this decomposition cannot be directly applied in matrix form for our label-aware cost $\tilde{\bm{C}}$. For example, $f_1(a) = a^2, f_2(b)=b^2, h_1(a)= a$, when the label indices $i_2$ and $j_2$ are the same, $h_2(b) = 2b$, otherwise $h_2(b) = -2b$. This leads to the tensor-matrix multiplication not being directly simplified into matrix-vector multiplication through this decomposition. However, the BUOT model requires the discriminative information introduced by label-aware cost $\tilde{\bm{C}}$. To realize matrix operations, we propose a new calculation method suitable for the cost $\tilde{\bm{C}}$, as shown in Thm.~\ref{thm:thm_1}.

\begin{thm}
\label{thm:thm_1}
Suppose there is a label indicator matrix $\bm{M}\in \mathbb{R}^{K\times K}$, the elements on its diagonal are all 1, and the remaining elements are all -1, then for the label-aware transport cost $\tilde{\bm{C}}$, we have
\begin{equation}
\left\{
\begin{aligned}
&\tilde{\bm{C}}(\bm{P}^s, \bm{P}^t)\otimes\bm{\Gamma}^1 = \bm{c}^1_{\bm{P}^s, \bm{P}^t}-2\bm{M}\odot( {\bm{P}^s}^T\bm{\Gamma}^1{\bm{P}^t}),\\
&\tilde{\bm{C}}(\bm{P}^s, \bm{P}^t)\otimes\bm{\Gamma}^2 = \bm{c}^2_{\bm{P}^s, \bm{P}^t}-2{\bm{P}^s}(\bm{M}\odot\bm{\Gamma}^2){\bm{P}^t}^T.
\end{aligned}
\right.
\end{equation}
where $\bm{c}^1_{\bm{P}^s, \bm{P}^t} = (\bm{P}^{sT})^2\bm{u}_1\bm{1}^T_{K}+ \bm{1}_{K}\bm{v}_1^{T}(\bm{P}^{t})^2$, $\bm{c}^2_{\bm{P}^s, \bm{P}^t} = (\bm{P}^s)^2\bm{u}_2\bm{1}^T_{n_t}+\bm{1}_{n_s}\bm{v}_2^{T}(\bm{P}^{tT})^2$, $\odot$ represent the Hadamard product.
\end{thm}

\begin{proof}
From the definition of tensor-matrix multiplication in Eq.~\eqref{eq:tensor-matrix}, we can obtain
\begin{equation}
\tilde{\bm{C}}(\bm{P}^s, \bm{P}^t)\otimes\bm{\Gamma}^1 = (\sum_{i_1j_1}\tilde{C}_{i_1j_1i_2j_2}\Gamma^1_{i_1j_1})_{i_2j_2}.
\end{equation}
As the cost $\tilde{\bm{C}}$ in Eq.~\eqref{eq:cost_mechanism} and decomposition in Eq.~\eqref{eq:decomposition}, the above equation can be written as
\begin{equation}
\begin{aligned}
&\left(\sum_{i_1j_1}\tilde{C}_{i_1j_1i_2j_2}\Gamma^1_{i_1j_1}\right)_{i_2j_2} \\
&= \left(\sum_{i_1j_1}((P_{i_1i_2}^s)^2+(P_{j_1j_2}^t)^2\pm2P_{i_1i_2}^sP_{j_1j_2}^t )\Gamma^1_{i_1j_1}\right)_{i_2j_2}\\
&= \left(\sum_{i_1j_1}(P_{i_1i_2}^s)^2 \Gamma^1_{i_1j_1}+\sum_{i_2j_2}(P_{j_1j_2}^t)^2 \Gamma^1_{i_1j_1} \pm \sum_{i_2j_2}2P_{i_1i_2}^sP_{j_1j_2}^t \Gamma^1_{i_1j_1}\right)_{i_2j_2}\\
&= \bm{A}^s+\bm{A}^t+\bm{A}^c,
\end{aligned}
\end{equation}
where
\begin{equation}
\left\{
\begin{aligned}
& A^s_{i_2j_2} =  \sum_{i_1}(P_{i_1i_2}^s)^2 \sum_{j_1}\Gamma^1_{i_1j_1} = ((\bm{P}^s)^2\bm{\Gamma}^1\bm{1}_{n_t})_{i_2},\\
& A^t_{i_2j_2} = \sum_{j_1}(P_{j_1j_2}^t)^2 \sum_{i_1}\Gamma^1_{i_1j_1} = 
(((\bm{\Gamma}^{1T}\bm{1}_{n_s})\bm{P}^t)^2)_{j_2},\\
& A^c_{i_2j_2}  = \pm(2\sum_{i_1}(P^s_{i_1i_2})\sum_{j_1}(P^t_{j_1j_2})\Gamma^1_{i_1j_1}).
\end{aligned}
\right.
\end{equation}
The cost $\tilde{C}$ can be viewed as follows: when $i_2 = j_2$, the cost remains unchanged, but when $i_2\neq j_2$, the cost can be seen as taking the opposite of the target prediction $P_{j_1j_2}^t$. Regardless of whether $i_2 = j_2$ or not, source term $\bm{A}^s$ is independent of the target prediction $P_{j_1j_2}^t$, and for target term $\bm{A}^t$, we have $(\pm P_{j_1j_2}^t)^2 = (P_{j_1j_2}^t)^2$. Therefore, $\bm{A}^s$ and $\bm{A}^t$ always remain unchanged. Since $ P_{i_1i_2}^s(-P_{j_1j_2}^t)\Gamma^1_{i_1j_1} = -(P_{i_1i_2}^sP_{j_1j_2}^t\Gamma^1_{i_1j_1})$, it can be viewed in cross term $\bm{A}^c$, $P_{i_1i_2}^sP_{i_2j_2}^t\Gamma^1_{i_1j_1}$ remains unchanged when $i_2 = j_2$, otherwise, $P_{i_1i_2}^sP_{j_1j_2}^t\Gamma^1_{i_1j_1}$ takes the opposite number. As $\bm{\Gamma}^1\bm{1}_{n_t}=\bm{u}_1, \bm{\Gamma}^{1T}\bm{1}_{n_s}=\bm{v}_1$, we have $\bm{A}^s = (\bm{P}^{sT})^2\bm{u}_1\bm{1}^T_{K}, \bm{A}^t =\bm{1}_{K}\bm{v}_1^{T}(\bm{P}^{t})^2, \bm{A}^c=2\bm{M}\odot( {\bm{P}^s}^T\bm{\Gamma}^1{\bm{P}^t})$, where the diagonal elements of label indicator matrix $\bm{M}$ are all 1, and the remaining elements are all -1. 

The decomposition of $\tilde{\bm{C}}(\bm{P}^s, \bm{P}^t)\otimes\bm{\Gamma}^2$ is similar as above,  but with a key difference in construction of the matrix $\bm{A}^c$. Since $ (-P^t_{j_1j_2})\Gamma^2_{i_2j_2} = P^t_{j_1j_2}(-\Gamma^2_{i_2j_2})$, it can be viewed in $\bm{A}^c$, $\Gamma^2_{i_2j_2}$ remains unchanged when $i_2 = j_2$, otherwise, $\Gamma^2_{i_2j_2}$ takes the opposite number. Then we have
$\bm{A}^s = (\bm{P}^s)^2\bm{u}_2\bm{1}^T_{n_t}, \bm{A}^t =\bm{1}_{n_s}\bm{v}_2^{T}(\bm{P}^{tT})^2, \bm{A}^c=2{\bm{P}^s}(\bm{M}\odot\bm{\Gamma}^2){\bm{P}^t}^T$.
\end{proof}

Thm.~\ref{thm:thm_1} simplifies the tensor-matrix multiplication with respect to $\tilde{\bm{C}}$ into matrix-vector multiplication. Without loss of generality, assume that $n_s=n_t=n$, the tensor-matrix multiplication has a training complexity of $O(n^2K^2)$. By decomposing the loss and converting the tensor-matrix multiplication to matrix calculation, the complexity can be reduced to $O(nK^2)$ for $\bm{\Gamma}_1$ and $O(n^2K)$ for $\bm{\Gamma}_2$. By converting the label indices of the source and target domains into the row and column indices of a matrix, different scenarios can be calculated.

\subsection{Model and Algorithm}

\begin{figure}[t]
    \centering
    \includegraphics[width=0.75\linewidth,trim=25 35 45 40,clip]{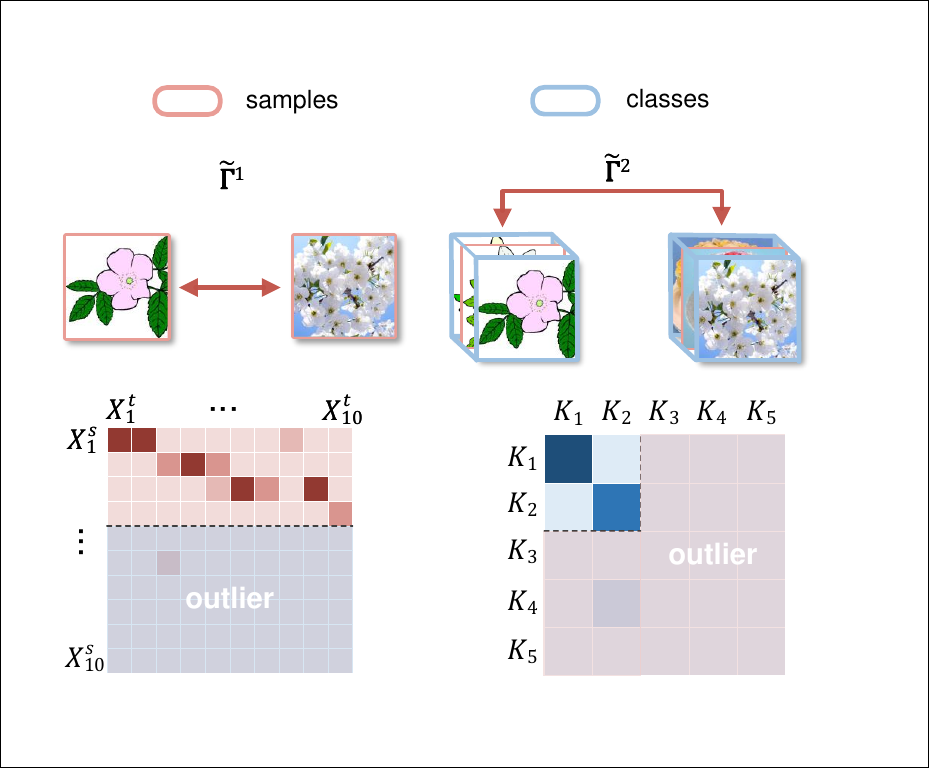}
    \caption{The structure of recovered transport. We hope that the recovered sample-level transport $\bm{\tilde{\Gamma}}^1$ primarily transports between shared class samples across domains, and the recovered class-level transport $\bm{\tilde{\Gamma}}^2$ should mainly occur between shared classes across domains. Best viewed in color.}
    \label{fig:Re}
\end{figure}

We can recover the expression for sample-wise and class-wise relations by using the bi-level transport plan learned from the BUOT model, which contains both sample-wise and class-wise information. Denote the recovered sample-level and class-level transport as $\bm{\tilde{\Gamma}}^1$ and $\bm{\tilde{\Gamma}}^2$. Fig.~\ref{fig:Re} is the structure illustration of $\bm{\tilde{\Gamma}}^1$ and $\bm{\tilde{\Gamma}}^2$. Referring to the task illustrated in Fig.~\ref{fig:Cost}, assume that the source domain has 5 classes, while the target domain has only 2 classes, that is $|\mathcal{Y}_{s}|=5$ and $|\mathcal{Y}_{t}|=2$. We hope that the recovered transport $\bm{\tilde{\Gamma}}^1$ primarily transports between shared class samples in the source domain and all samples in the target domain, i.e., between the first 4 samples in the source domain and all 10 samples in the target domain. Meanwhile, the recovered transport $\bm{\tilde{\Gamma}}^2$ should primarily occur between the shared classes in the source domain and target domain, i.e., between the classes $K_1,K_2$ in the source domain and target domain.

Through the BUOT model, we can jointly learn the relations between samples and between classes. However, if we want to separately consider the transport relations between samples and classes, we can leverage the class-level and sample-level transport to recover the explicit expression of independent sample-wise and class-wise relations. As class-level transport can provide discriminative information for sample-level transport to identify outlier classes, we leverage the relations learned from $\bm{\Gamma}^2$ between classes to weight $\bm{\Gamma}^1$. We aim to map the class-level transport $\bm{\Gamma}^2$ into a weight matrix $\bm{W}^1\in \mathbb{R}^{n_s\times n_t}$, where $W_{i_1j_1}^1$ represents the transport value from the label of source sample $\bm{x}_{i_1}^s$ to the label of target sample $\bm{x}_{j_1}^t$ in $\bm{\Gamma}^2$. Hence, each element of the weight matrix $\bm{W}^1$ is derived from the transport plan $\bm{\Gamma}^2$ and the sample labels. Here, the labels of source samples are determined by the ground-truth labels and the labels of target samples are determined by the pseudo labels. Define source indicator matrix $\bm{S}\in \mathbb{R}^{n_s\times K}$ and target indicator matrix $\bm{T}\in \mathbb{R}^{n_t\times K}$. $\bm{S}$ indicates the relations from source samples to classes, where $S_{i_1i_2}=1$ if the $i_1$-th source sample belongs to $i_2$-th class, otherwise $ S_{ i_1i_2}=0$. Similarly, $\bm{T}$ indicates the relations from target samples to classes, where $T_{ i_1i_2}=1$ if the $i_1$-th target sample belongs to $i_2$-th class, otherwise $ T_{ i_1i_2}=0$. Then the sample-level weight matrix $\bm{W}^1$ can be computed as $\bm{W}^1 = \bm{S}\bm{\Gamma}^2\bm{T}^T$. $\bm{S}\bm{\Gamma}^2$ is obtained by mapping $\bm{\Gamma}^2$ through the source indicator matrix, and $\bm{S}\bm{\Gamma}^2\bm{T}^T$ considers the relations between the target samples and classes based on $\bm{S}\bm{\Gamma}^2$. Then multiply this weight matrix  $\bm{W}^1$ by $\bm{\Gamma}_1$ can obtain the recovered sample-level transport $\bm{\tilde{\Gamma}}^1$ 
\begin{equation}
\bm{\tilde{\Gamma}}^1 =  \bm{W}^1 \odot \bm{\Gamma}^1.
\end{equation}

Also as sample-level transport can provide essential structure information for class-level transport, we leverage the sample-wise relations learned from $\bm{\Gamma}^1$ to weight $\bm{\Gamma}^2$. We sum the transport plan values of $\bm{\Gamma}^1$ that belong to the same class to obtain a class-level weight matrix $\bm{W}^2\in \mathbb{R}^{K\times K}$. The formula for $\bm{W}^2$ is $\bm{W}^2 = \bm{S}^T\bm{\Gamma}^1\bm{T}$.
$\bm{S}^T\bm{\Gamma}^1$ sums the transport plan values in $\bm{\Gamma}^1$ based on the class of source sample, and $\bm{S}^T\bm{\Gamma}^1\bm{T}$ sums the transport plan values in $\bm{S}^T\bm{\Gamma}^1$ based on the
class of target sample. Multiply $\bm{W}^2$ by $\bm{\Gamma}^2$ can obtain the recovered class-level transport $\bm{\tilde{\Gamma}}^2$ 
\begin{equation}
\bm{\tilde{\Gamma}}^2 =  \bm{W}^2 \odot \bm{\Gamma}^2.
\end{equation}

Since $\bm{\tilde{\Gamma}}^2$ contains both rich sample-wise and class-wise information, the bi-level weights $\bm{\omega}$ can be obtained by summing the rows of $\bm{\tilde{\Gamma}}^2$, i.e. $\bm{\omega} = \sum_{i=1}^K\bm{\tilde{\Gamma}}^2_{ij}$. The weights $\bm{\omega}$ represent the transport relations from the overall source domain to each class in the target domain. If $\omega_i$ is particularly small, it suggests that the $i$-th class in the source domain is rarely transported to the target domain, which means that this class is likely to be an outlier class in the source domain. Conversely, if $\omega_i$ is relatively large, it implies the $i$-th class exists in both the source and target domains, and is probably a shared class in the source domain.

To ensure the accuracy of prediction, it is usually necessary to train the predictor by reducing the classification error of the source domain. Suppose $\bm{Y}^s$ is a label matrix composed of one-hot labels in the source domain, the cross-entropy loss is 

\begin{equation}
\label{eq:cross-entropy}
\mathcal{L}_{CE} = \sum_{i=1}^{n_s}\sum_{j=1}^{K} - Y^s_{ij}\log P^s_{ij}.
\end{equation}

Further, to make the source label distribution close to the target label distribution and reduce the impact of outlier class samples, we use $\bm{\omega}$ to weigh the source domain cross-entropy loss. The reweighted cross-entropy loss is formulated as

\begin{equation}
\label{eq:r_cross-entropy}
\mathcal{L}_{RCE} = \sum_{i=1}^{n_s}\sum_{j=1}^{K} - \omega_jY^s_{ij}\log P^s_{ij}.
\end{equation}

Target domain entropy reflects the uncertainty of the model predictions in the target domain. If the target domain entropy is high, it indicates that the model predictions in the target domain are highly uncertain. To enhance the certainty of pseudo labels, we minimize the target domain entropy as 
\begin{equation}
\label{eq:target-entropy}
\mathcal{L}_{Ent} = \sum_{i=1}^{n_t}\sum_{j=1}^{K} - P^t_{ij}\log P^t_{ij}.
\end{equation}

\begin{algorithm}[H]
\small
    \caption{BUOT for PDA}
    \label{alg:BUOTforPDA}
    \begin{algorithmic}[1] 
    \REQUIRE Source dataset $\mathcal{D}_s$, Target dataset $\mathcal{D}_t$, Warming up iterations $T_{\text{warm}}$, UOT iterations $T_{\text{UOT}}$, Max iterations $T_{\max}$, Learning rate $\eta$;\\
    \ENSURE Representation learner $f$, Task predictor $h$;
    \STATE Initialize the network parameters $\bm{\theta}=\{\bm{\theta}_f,\bm{\theta}_h\}$; \\
    \% \textit{Warming Up Stage}
    \FOR {$iter=1,...,T_{\text{warm}}$}
    \STATE Sample a batch from $\mathcal{D}_s$;
    \STATE Forward propagate batch data according to $\bm{p}^s=h(f(\bm{x}^s))$; \\
    \STATE Compute $\mathcal{L}_{CE}$ as Eq.~\eqref{eq:cross-entropy};
    \STATE Update network parameters $\bm{\theta}$ with learning rate $\eta$;
    \ENDFOR\\
    \% \textit{BUOT Learning Stage}
    \FOR {$iter=T_{\text{warm}}+1,...,T_{\max}$}
    \STATE Sample a batch from $\mathcal{D}_s$ and $\mathcal{D}_t$;
    \STATE Forward propagate batch data, and compute the predictions $\bm{p}^s$ and $\bm{p}^t$;
    \STATE Initialize $\bm{\Gamma}^1=\mu_1 \nu_1^T$, $\bm{\Gamma}^2=\mu_2 \nu_2^T$ ; \\
    \FOR {$iter=1,...,T_{\text{UOT}}$}
    \STATE Update $\bm{\Gamma}^1$ by $UOT(\mu_1,\nu_1,\tilde{\bm{C}}(\bm{P}^s, \bm{P}^t)\otimes\bm{\Gamma}^2)$; \\
    \STATE Update $\bm{\Gamma}^2$ by $UOT(\mu_2,\nu_2,\tilde{\bm{C}}(\bm{P}^s, \bm{P}^t)\otimes\bm{\Gamma}^1)$;
    \ENDFOR\\
    \STATE Compute $\bm{\tilde{\Gamma}}^2$ with current $\bm{\Gamma}^1$ and $\bm{\Gamma}^2$;
    \STATE Compute bi-level weights $\bm{\omega}$ with $\bm{\tilde{\Gamma}}^2$ and compute reweighted cross-entropy loss $\mathcal{L}_{RCE}$ as Eq.~\eqref{eq:r_cross-entropy};
    \STATE Compute $\mathcal{L}_{BUOT}$ with current $\bm{\Gamma}^1$, $\bm{\Gamma}^2$ as Eq.~\eqref{eq:BUOT_loss};
    \STATE Compute the overall objective $\mathcal{L}$ as Eq.~\eqref{eq:all};
    \STATE Update network parameters $\bm{\theta}$ with learning rate $\eta$;
    \ENDFOR
    \end{algorithmic}
\end{algorithm}
The BUOT loss can be written as
\begin{equation}\label{eq:BUOT_loss}
\mathcal{L}_{BUOT}= \langle \mathbf{C}(\bm{P}^s, \bm{P}^t)\otimes\bm{\Gamma}^1, \bm{\Gamma}^2 \rangle_F - H(\bm{\Gamma}^1, \bm{\Gamma}^2)+ D(\bm{\Gamma}^1, \bm{\Gamma}^2).
\end{equation}

Our final loss is as follows
\begin{equation}
\label{eq:all}
\mathcal{L} = \mathcal{L}_{Cls}+\lambda\mathcal{L}_{BUOT},
\end{equation}
where $\lambda$ is the trade-off parameters, $\mathcal{L}_{Cls}=\mathcal{L}_{RCE}+\lambda_t \mathcal{L}_{Ent}$ is the classification loss.

For optimization, we use an iterative approach as shown in Alg.~\ref{alg:BUOTforPDA}. Since we need to use the prediction vectors for all samples, we obtain a relatively accurate predictor only through the source cross-entropy loss in the warming up stage. In the BUOT learning stage, first, the network parameters $\bm{\theta}_f, \bm{\theta}_h$ are fixed, and the optimal transport plan $\bm{\Gamma}^1$ and $\bm{\Gamma}^2$ are learned by the scaling algorithm~\cite{chizat2018scaling}; then the transport plan $\bm{\Gamma}^1$ and $\bm{\Gamma}^2$ are fixed, and the network parameters $\bm{\theta}_f, \bm{\theta}_h$ are optimized by the Adam optimizer~\cite{kingma2014adam}.

When computing the BUOT loss, tensor-matrix multiplication is required. To improve computational efficiency, we simplify it to matrix-vector multiplication based on the Thm.~\ref{thm:thm_1}. We iteratively compute $\bm{\Gamma}^1$ and $\bm{\Gamma}^2$. Specifically, when calculating $\bm{\Gamma}^1$, we fix $\bm{\Gamma}^2$, reducing the problem to a standard UOT problem: $UOT(\mu_1,\nu_1,\tilde{\bm{C}}(\bm{P}^s, \bm{P}^t)\otimes\bm{\Gamma}^2)$. This can be efficiently solved using the scaling algorithm~\cite{chizat2018scaling}. Subsequently, we fix $\bm{\Gamma}^1$ and calculate $\bm{\Gamma}^2$.

\section{Experiments}   \label{chap:4}

\subsection{Datasets}
We demonstrate the effectiveness of BUOT using four standard PDA datasets. These datasets are described as follows

\textbf{ImageCLEF}~\cite{caputo2014imageclef} is a popular standard benchmark for PDA problem including three domains: Caltech-256 (C), ImageNet ILSVRC 2012 (I) and Pascal VOC 2012 (P). Each domain includes 12 classes and each class contains 50 images. There are six PDA tasks to be evaluated.

\textbf{Office-31}~\cite{saenko2010adapting} is also a standard benchmark dataset for PDA problem. It contains 4,110 images collected from three various domains: Amazon Website (A), Web camera (W) and Digital SLR camera (D). A, W and D share the same label space with 31 classes. We evaluate all methods on six PDA tasks.

\textbf{VisDA-2017}~\cite{peng2017visda} is a synthetic-to-real image dataset consisting of two domains: synthetic and real images. It has images of 12 classes, including 152,397 synthetic images and 55,388 real images. We take the synthetic images as the source domain and the real images as the target domain.

\textbf{Office-Home}~\cite{venkateswara2017deep} is a challenging dataset that collects images of everyday objects with four domains: Artistic (Ar), Clipart (Cl), Product (Pr) and Real-World (Rw). Each domain contains 65 classes and they amount to around 15,500 images. We evaluate our method in all 12 tasks.

\subsection{Implementation Details}
 In PDA problem, there are only 6 classes in the target domain on the ImageCLEF dataset, and therefore six transfer tasks$\colon$ I12$\rightarrow$ P6, P12$\rightarrow$I6, I12$\rightarrow$C6, C12$\rightarrow$I6, C12$\rightarrow$P6, P12$\rightarrow$C6. Similarly, in Office-31, the source domain has 31 classes while the target domain has only 10 classes. In VisDA-2017, we consider the following transport task$\colon$ S12$\rightarrow$ R6. In Office-Home, we choose the first 25 classes for the shared classes. The Adam optimizer is used for gradient descent-based optimization. All datasets employ pre-trained ResNet-50 networks. In the comparison experiments, the classification accuracy on the target domain is used as the evaluation metric. For each transfer task, we report the average accuracy over five random repeats.

\subsection{Experiment Results and Analysis}
\subsubsection{Comparison with state-of-the-art (SOTA) PDA methods}
\begin{table}[!t]
    \scriptsize
    \centering
    \caption{Accuracies (\%) on ImageCLEF (ResNet-50) for PDA.}
    \label{tab:pda-acc-imageclef}
    \renewcommand{\tabcolsep}{0.61pc} 
     \begin{tabular}{c|ccccccc}
         \toprule
           ImageCLEF&  I$\rightarrow$P& P$\rightarrow$I& I$\rightarrow$C& C$\rightarrow$I& C$\rightarrow$P& P$\rightarrow$C& \textbf{Avg} \\
         \midrule
         Source-only~\cite{he2016deep}           &  78.3& 86.9 &91.0 &84.3 &72.5& 91.5 &84.1 \\
         DANN~\cite{ganin2016domain}       &  78.1& 86.3 &91.3& 84.0 &72.1 &90.3 &83.7 \\
         PADA~\cite{cao2018partialada}  &  81.7& 92.1& 94.6 &89.8& 77.7& 94.1& 88.3 \\
         SAFN~\cite{xu2019larger}  & 79.5 &90.7& 93.0& 90.3 &77.8 &94.0& 87.5 \\
         DMP~\cite{luo2022unsupervised} &  82.4& 94.5 &96.7 &94.3& 78.7 &96.4& 90.5 \\
         Ma et al.~\cite{ma2024small} & 86.7 & 92.0 & 97.0 & 89.3& 83.3 &97.0 & 90.9 \\
         \midrule
          \rowcolor{gray!25}
         \textbf{BUOT}                           & \textbf{91.0} & \textbf{94.7} & \textbf{98.0}& \textbf{94.7} & \textbf{90.3} & \textbf{99.3}& \textbf{94.7}\\
          \bottomrule
    \end{tabular}
\end{table}

We compare our method with the following SOTA methods$\colon$ Source-only~\cite{he2016deep}, DANN~\cite{ganin2016domain}, PADA~\cite{cao2018partialada}, SAFN~\cite{xu2019larger}, DRCN~\cite{li2020deep}, B$\mathrm{A}^3$US~\cite{liang2020balanced}, TSCDA~\cite{ren2020learning}, AR~\cite{gu2021adversarial}, DMP~\cite{luo2022unsupervised}, DARL~\cite{chen2022domain}, AGAN~\cite{kim2022adaptive}, Lin et al.~\cite{lin2022adversarial}, CSDN~\cite{li2023critical}, SAN++~\cite{cao2023big}, RAN~\cite{wu2023reinforced}, IDSP~\cite{li2023partial}, CLA~\cite{yang2023contrastive}, SLM~\cite{sahoo2023select}, Ma et al.~\cite{ma2024small}.
The results of BUOT on ImageCLEF, Office-31, VisDA-2017 and Office-Home are shown in Tab.~\ref{tab:pda-acc-imageclef},~\ref{tab:pda-acc-office31} and~\ref{tab:pda-acc-officehome} respectively.

\begin{table*}[!t]
    \scriptsize
    \centering
    \caption{Accuracies (\%) on Office-31 (ResNet-50) and VisDA-2017 (ResNet-50) for PDA.}
    \label{tab:pda-acc-office31}
    \renewcommand{\tabcolsep}{0.1pc} 
     \begin{tabular}{c|ccccccc|c}
         \toprule
           \multirow{2} {*} {Methods} & \multicolumn{7}{c|} {Office-31} & VisDA-2017 \\
           &  A$\rightarrow$W& D$\rightarrow$W& W$\rightarrow$D& A$\rightarrow$D& D$\rightarrow$A& W$\rightarrow$A& \textbf{Avg} &  S$\rightarrow$R \\
         \midrule
         Source-only~\cite{he2016deep}           &  75.6 &96.3 &98.1& 83.4& 83.9 &85.0 &87.1 &  45.3 \\
         DANN~\cite{ganin2016domain}         &   73.6 &96.3 &98.7& 81.5 &82.8 &86.1& 86.5 & 51.0 \\
         PADA~\cite{cao2018partialada}   &   86.5 &99.3& \textbf{100.0}& 82.2& 92.7& 95.4& 92.7 &  53.5\\
         SAFN~\cite{xu2019larger}  &  87.5 &96.6 &99.4& 89.8& 92.6 &92.7& 93.1 & 67.7 \\
         DRCN~\cite{li2020deep}& 90.8& \textbf{100.0} & \textbf{100.0}& 94.3& 95.2 &94.8& 95.9 &58.2 \\
         B$\mathrm{A}^3$US~\cite{liang2020balanced} & 99.0 & \textbf{100.0} & 98.7 & 99.4 & 94.8 & 95.0 & 97.8 & -\\
         TSCDA~\cite{ren2020learning} & 96.8 &\textbf{100.0} &\textbf{100.0} & 98.1 & 94.8 & 96.0 & 97.6 & - \\
         AR~\cite{gu2021adversarial} & 93.5 &\textbf{100.0}& 99.7& 96.8& 95.5& 96.0& 96.9 &  88.7 \\
         DMP~\cite{luo2022unsupervised}&  96.6 &\textbf{100.0} &\textbf{100.0} &96.4& 95.1& 95.4& 97.2 & 72.7 \\
         DARL~\cite{chen2022domain}  &  94.6 & 99.7 & \textbf{100.0} &  98.7 & 94.6 & 94.3 & 97.0 & 67.8 \\
         AGAN~\cite{kim2022adaptive} &97.3  &\textbf{100.0} &\textbf{100.0}& 94.3 & 95.7 &95.7 & 97.2 & 67.7 \\
         Lin et al.~\cite{lin2022adversarial} & 99.7 & \textbf{100.0} & \textbf{100.0} & 96.8 &  96.1 & 96.6 & 98.2 & 69.8 \\
         CSDN~\cite{li2023critical} & 98.9 & \textbf{100.0} & \textbf{100.0} & 98.7 & 94.3 & 94.6 & 97.8 & 67.6\\
         SAN++~\cite{cao2023big}  & 99.7 & \textbf{100.0} & \textbf{100.0} & 98.1 &  94.1 & 95.5 & 97.9 & 63.1\\
         RAN~\cite{wu2023reinforced} &99.0 & \textbf{100.0} & \textbf{100.0} &  97.7& 96.3 & 96.2 & 98.2 &75.1 \\
         IDSP~\cite{li2023partial}  &  99.7 & 99.7 &\textbf{100.0} & 99.4 & 95.1 &95.7 & 98.3 & - \\
         CLA~\cite{yang2023contrastive} & \textbf{100.0} & \textbf{100.0} & \textbf{100.0} &  \textbf{100.0} & 94.5& 96.7 &98.5& -\\
         SLM~\cite{sahoo2023select} & 99.8  &\textbf{100.0} & 99.8 &  98.7 & 96.1 & 95.9 & 98.4 & 91.7 \\
         Ma et al.~\cite{ma2024small} &94.6 & 91.7 & 94.1 & 99.4 & 94.1 & 98.7 & 95.4 & - \\
         \midrule
          \rowcolor{gray!25}
         \textbf{BUOT}       & \textbf{100.0} & \textbf{100.0} & \textbf{100.0}& \textbf{100.0} & \textbf{97.8} & \textbf{98.4}& \textbf{99.5}& \textbf{93.3}\\
          \bottomrule
    \end{tabular}
\end{table*}

\textbf{ImageCLEF}. Tab.~\ref{tab:pda-acc-imageclef} shows the results of BUOT for six transfer tasks on ImageCLEF. Since PADA uses adversarial networks to solve the PDA problem, it introduces class-level weights based on DANN. The accuracy is improved from 83.7\% to 88.3\%, which shows that weighting the source domain can effectively improve the classification accuracy. Similarly weighting the source domains, DMP further takes into account the sample-wise information resulting from the manifold alignment, thus the accuracy improvement of 2.2\% compared to PADA. Compared to other methods, BUOT can obtain more accurate class weights by simultaneously considering both class-wise and sample-wise information. This allows BUOT to learn the bi-level weights with fewer errors, which in turn enables it to more correctly identify the outlier classes. Thus, the average accuracy of BUOT surpasses all the comparison methods, reaching 94.7\%.

\textbf{Office-31}. Tab.~\ref{tab:pda-acc-office31} (left) shows the results of BUOT for six transfer tasks on Office-31. We observe that the results of BUOT are better than other methods with average accuracy of 99.5\%. In fact, BUOT even reaches 100\% accuracy on four transfer learning tasks. Compared with DRCN and TSCDA, which also use weights to weight the source domain, the accuracy of BUOT increases by 3.6\% and 1.9\%, respectively. That's because we use bi-level unbalanced optimal transport method to learn the correspondence between the source and target domains. Then we can obtain the transport relations between the shared classes in source and target domains. Compared with methods such as B$\mathrm{A}^3$US and DARL that do not use weights, the accuracy of BUOT is improved by 1.7\% and 2.5\%, respectively. This shows that using sample-wise and class-wise information to weight the source domain can help improve classification accuracy. Weighting the source domain can reduce the negative impact of erroneous information compared to selecting certain source domain samples.

\begin{table*}[!t]
    \scriptsize
    \centering
    \caption{Accuracies (\%) on Office-Home (ResNet-50) for PDA.}
    \label{tab:pda-acc-officehome}
    \renewcommand{\tabcolsep}{0pc} 
      \begin{tabular}{c|ccccccccccccc}
         \toprule
         Office-Home& Ar$\rightarrow$Cl& Ar$\rightarrow$Pr& Ar$\rightarrow$Rw& Cl$\rightarrow$Ar& Cl$\rightarrow$Pr& Cl$\rightarrow$Rw& Pr$\rightarrow$Ar& Pr$\rightarrow$Cl& Pr$\rightarrow$Rw& Rw$\rightarrow$Ar& Rw$\rightarrow$Cl& Rw$\rightarrow$Pr& \textbf{Avg}\\
         \midrule
         Source-only \cite{he2016deep} &  46.3& 67.5 &75.9& 59.1 &59.9& 62.7& 58.2 &41.8 &74.9& 67.4 &48.2& 74.2& 61.4\\
         DANN \cite{ganin2016domain}  & 43.8& 67.9 &77.5 &63.7 &59.0 &67.6 &56.8& 37.1 &76.4 &69.2& 44.3& 77.5 &61.7\\
         PADA~\cite{cao2018partialada}  & 52.0 &67.0& 78.7& 52.2& 53.8& 59.0& 52.6 &43.2& 78.8 &73.7 &56.6& 77.1& 62.1\\
         SAFN~\cite{xu2019larger}  & 58.9 &76.3& 81.4 &70.4 &73.0 &77.8& 72.4 &55.3 &80.4& 75.8& 60.4 &79.9 &71.8\\
         DRCN~\cite{li2020deep} & 51.6 &75.8& 82.0& 62.9 &65.1& 72.9& 67.4& 50.0 &81.0& 76.4 &57.7 &79.3& 68.5\\
         B$\mathrm{A}^3$US~\cite{liang2020balanced} & 60.6 &83.2& 88.4& 71.8 &72.8 &83.4& 75.5& 61.6& 86.5& 79.3& 62.8 &86.1 &76.0\\
         TSCDA~\cite{ren2020learning} & 63.6 & 82.5& 89.6& 73.7 &  73.9 & 81.4 & 75.4 & 61.6 & 87.9 &\textbf{ 83.6} & \textbf{67.2} & 88.8 &  77.4\\
         AR~\cite{gu2021adversarial} & \textbf{67.4}& 85.3& 90.0 &\textbf{77.3} &70.6 &85.2& \textbf{79.0} &64.8& 89.5& 80.4 &66.2& 86.4 &78.3\\
         DMP~\cite{luo2022unsupervised} & 59.0& 81.2 &86.3 &68.1& 72.8 &78.8& 71.2 &57.6 &84.9& 77.3& 61.5 &82.9& 73.5\\
         DARL~\cite{chen2022domain}  &55.3 & 80.7 & 86.4 & 67.9 & 66.2 & 78.5 & 68.7 &50.9 &87.8 & 79.5 &57.2&85.6 & 72.1\\
         AGAN~\cite{kim2022adaptive} &56.4 & 77.3 & 85.1 & 74.2 & 73.8 & 81.1 & 70.8 & 51.5 & 84.5 & 79.0 & 56.8 & 83.4 & 72.8 \\
         CSDN~\cite{li2023critical}  & 57.3 & 78.1 & 87.0& 71.0 & 70.1 & 79.0 & 75.8 & 54.9   & 86.0 & 79.6 & 61.3 & 84.7 & 73.7\\
         SAN++~\cite{cao2023big} & 61.3 &  81.6 &  88.6 & 72.8 & 76.4 &  81.9 & 74.5 & 57.7 &  87.2 &  79.7 & 63.8 & 86.1 & 76.0 \\
         RAN~\cite{wu2023reinforced} & 63.3 & 83.1 & 89.0 & 75.0 & 74.5 & 83.0 & 78.0 & 61.2 & 86.7 & 79.9 & 63.5 & 85.0 & 76.8 \\
         IDSP~\cite{li2023partial} & 60.8 & 80.8 & 87.3 & 69.3 &76.0& 80.2 & 74.7 &59.2 & 85.3 & 77.8 & 61.3 & 85.7 & 74.9 \\
         CLA~\cite{yang2023contrastive} & 66.7 & 85.6&90.9&75.6& 76.9& 86.8 &78.8 &\textbf{67.4}& 88.7&81.7&66.9&87.8& \textbf{79.5} \\
         SLM~\cite{sahoo2023select} & 61.1 & 84.0 & \textbf{91.4} & 76.5 & 75.0 & 81.8 & 74.6 & 55.6 & 87.8 & 82.3 & 57.8 & 83.5 & 76.0\\
         Ma et al.~\cite{ma2024small} & 60.6 & 75.2 & 85.3 & 67.4 & 66.8 & 77.1 & 70.2 & 58.0 &84.7 & 74.2 &53.7 & 81.2 & 71.2 \\
         \midrule
         \rowcolor{gray!25}
         \textbf{BUOT}& 56.8 & \textbf{88.3} & 90.1 & 66.9 & \textbf{84.3}& \textbf{87.5} &  75.1 & 60.4 & \textbf{89.8}& 75.4& 55.0& \textbf{89.7} & 76.6\\
         \bottomrule
    \end{tabular}
\end{table*}

\textbf{VisDA-2017}. Tab.~\ref{tab:pda-acc-office31} (right) shows the results of BUOT on VisDA-2017. VisDA-2017 presents a more challenging scenario because it has larger sample size compared to ImageCLEF and Office-31. We consider the more challenging task of transferring from the synthetic (S) domain to the real (R) domain. Due to the inherent difficulty of the S to R task, the accuracy of most methods is observed to be less than 80\%. Notably, only AR and SLM exceed 80\%, achieving 88.7\% and 91.7\%, respectively. However, BUOT outperforms all comparison methods with 93.3\%. This result demonstrates the effectiveness of BUOT even when dealing with large-scale datasets, highlighting its superior performance in handling complex transfer learning tasks.

\textbf{Office-Home}. Tab.~\ref{tab:pda-acc-officehome} shows the results of BUOT for 12 transfer tasks on Office-Home. Compared to the ImageCLEF and Office-31 datasets, the Office-Home dataset has more classes and includes four distinct domains with significant differences between them, making it a more challenging dataset for cross-domain knowledge transfer. For the Source-only model, the accuracy of several transfer tasks is less than 50\%, such as Ar$\rightarrow$Cl and Pr$\rightarrow$Cl. Our proposed BUOT achieved an accuracy of 76.6\%, which is slightly lower than CLA. However, it is noteworthy that CLA employs data augmentation strategies such as cropping and recoloring on the original data, whereas our method does not utilize any data augmentation techniques. IDSP indicates that incorrect domain alignment can lead to negative transfer, and thus need avoid domain alignment.
However, our method can effectively aligns domains, reducing cross-domain discrepancies and achieving a 1.7\% improvement over IDSP.

\begin{table}[t]
    \scriptsize
\centering
\caption{Results of ablation study.}
\label{tab:ablation}
\renewcommand{\tabcolsep}{0.1pc} 
\begin{tabular}{cc|cccc}
\toprule
\multicolumn{2}{c|}{Objective} & \multirow{2}{*}{ImageCLEF} & \multirow{2}{*}{Office-31} &  \multirow{2}{*}{VisDA-2017}  &  \multirow{2}{*}{Office-home}\\
 $\omega$             & $\mathcal{L}_{BUOT}$ &                                &                                   &       &   \\
 \midrule
 \checkmark            & $\times$   &          91.6                    &           98.2                &         92.7   &  72.8  \\
 $\times$  & \checkmark     &      89.0            &      98.3                   &         92.3    & 72.2 \\
 \midrule
 \rowcolor{gray!25}
\checkmark &\checkmark             &         \textbf{94.7}                      &       \textbf{99.5}                   &        \textbf{93.3}     &   \textbf{76.6} \\
\bottomrule
\end{tabular}
\end{table}

\subsubsection{Ablation Study}

We analyze the effectiveness of each module of BUOT through the ablation experiments, the results are shown in Tab.~\ref{tab:ablation}. It can be seen that only applying the bi-level weights to the source cross-entropy loss can still ensure performance improvement on each dataset. When we only consider the BUOT loss without weighting the source domain cross-entropy loss, although the alignment between the source and target domains is achieved, the outlier class samples from the source domain are likely matched to the target domain. Therefore, as shown in the second row, although the accuracy on each dataset has improved, the overall performance is still inferior to the case where only weight the source domain cross-entropy loss. Thus in the PDA problem, while aligning the source and target domains is important, identifying the outlier classes in the source domain is even more crucial.

\begin{figure}[t]
    \centering
    \subfloat[ablation with OT\label{fig:abla_ot}]{\includegraphics[width=0.4\linewidth,trim=160 140 195 125,clip]{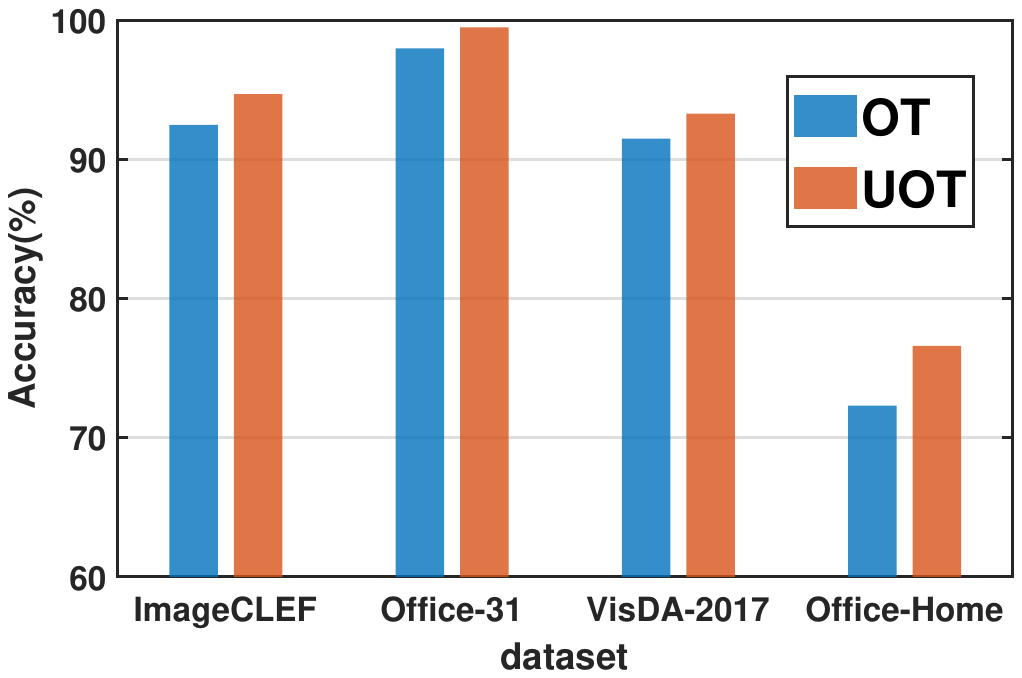}}
    \hfill
    \subfloat[ablation with cost\label{fig:abla_cost}]{\includegraphics[width=0.4\linewidth,trim=160 140 195 125,clip]{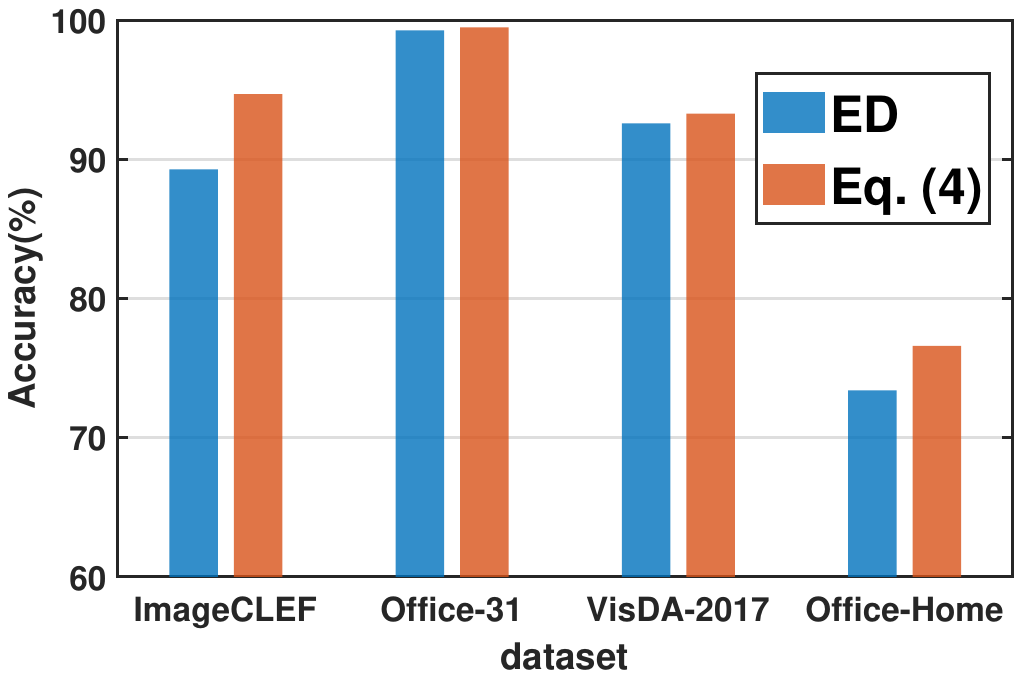}}
    \hfill\\
    \subfloat[ablation with Thm.~\ref{thm:thm_1}\label{fig:abla_th1}]{\includegraphics[width=0.4\linewidth,trim=160 140 195 125,clip]{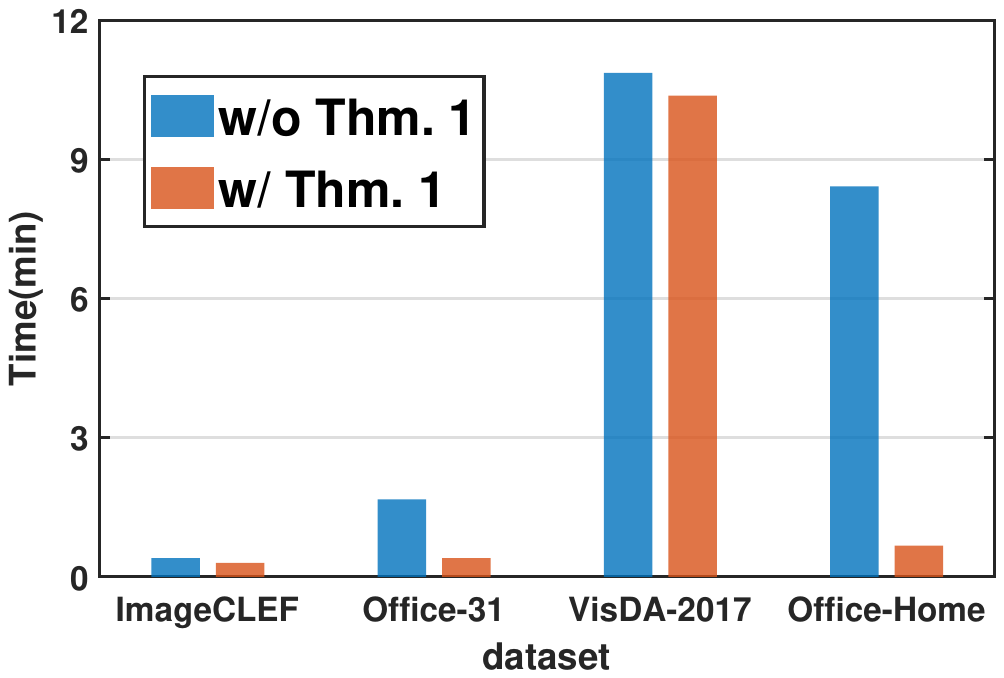}}
    \hfill
    \subfloat[ablation with Thm.~\ref{thm:thm_1}\label{fig:abla_th131}]{\includegraphics[width=0.4\linewidth,trim=160 140 195 125,clip]{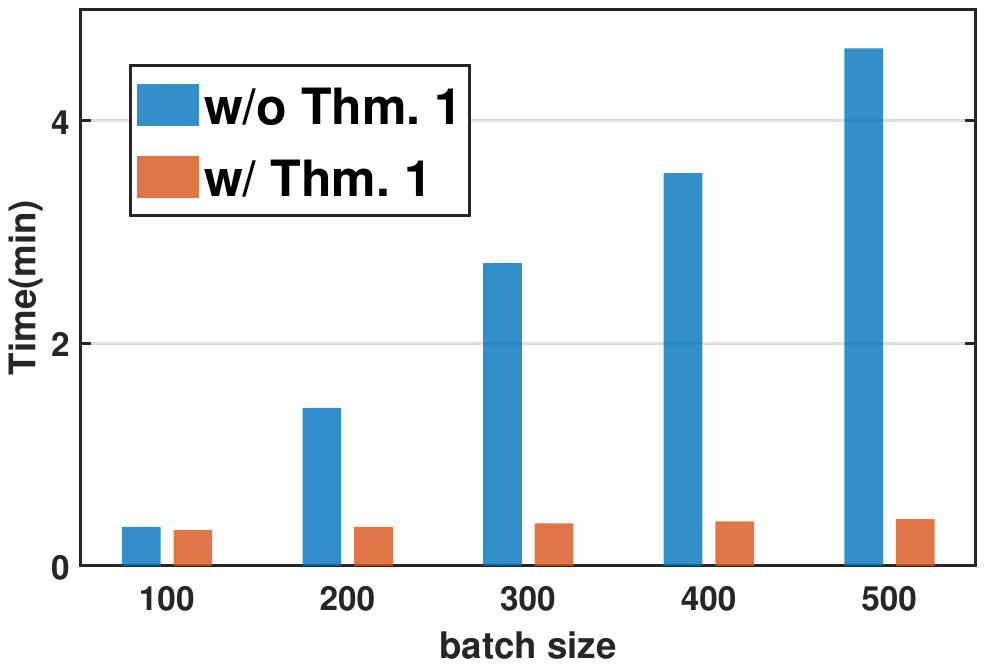}}
    \hfill
    \\
    \caption{(a)-(c): Ablation with OT, cost function and Thm.~\ref{thm:thm_1} on different datasets, where w/o and w/ denote without and with, respectively. (d):Ablation with Thm.~\ref{thm:thm_1} on Office-31 A$\rightarrow$W with varying batch size. Best viewed in color.}
\end{figure}

In order to evaluate the rationality of choosing UOT in the PDA problem, we replace UOT with OT in the BUOT model, and the experimental results are shown in Fig.~\subref*{fig:abla_ot}. From Fig.~\subref*{fig:abla_ot}, we can see that replacing UOT with OT leads to a decrease in accuracy for all datasets. This decline is particularly notable on the Office-Home dataset, where the accuracy of UOT is about 4\% higher than that of OT. This indicates that the relaxation of the marginal constraints in UOT makes it more suitable for addressing the PDA problem compared to OT. 

We also compare the label-aware cost Eq.~\eqref{eq:cost_mechanism} with the squared Euclidean distance. Noting the squared Euclidean distance as ED, the results are shown in Fig.~\subref*{fig:abla_cost}. It can be seen that the accuracy of the BUOT model with label-aware cost is higher, which suggests that more discriminative information can be learned using label-aware cost and helps to recognize the outlier classes. To validate that the matrix-vector multiplication proposed in Thm.~\ref{thm:thm_1} indeed improves computational efficiency, we measure the runtime of implementations with and without Thm.~\ref{thm:thm_1}. Tensor-matrix multiplication is employed when not using Thm.~\ref{thm:thm_1}. Specifically, to control variables, only one task (the first task in the comparison experiments) is run on each dataset, with identical batch sizes and max iterations. As shown in Fig.~\subref*{fig:abla_th1}, large-scale datasets like VisDA-2017 and Office-Home require more time. Additionally, since the batch size is the same, datasets with higher class counts exhibit larger efficiency gap between with and without Thm.~\ref{thm:thm_1}. Notably, Office-Home (with many classes) demonstrates significant time savings with Thm.~\ref{thm:thm_1}, whereas VisDA-2017 (with fewer classes) shows minimal differences despite its long running times. Furthermore, on the Office-31 dataset with fixed class numbers, we vary the batch size from 100 to 500, as shown in Fig.~\subref*{fig:abla_th131}. As the batch size increases, the runtime without Thm.~\ref{thm:thm_1} grows significantly, while the runtime with Thm.~\ref{thm:thm_1} remains stable. These findings demonstrate that Thm.~\ref{thm:thm_1} significantly improves computational efficiency.

\subsubsection{Generalization ability}

\begin{table}[t]
 \scriptsize
\centering
\caption{Risk on Office-31 A$\rightarrow$W.}
\label{tab:risk}
\renewcommand{\tabcolsep}{1pc} 
\begin{tabular}{c|ccc}
 \toprule
 Methods  & $\hat{\varepsilon}_s$  &$\varepsilon_t$  &  $| \hat{\varepsilon}_s -\varepsilon_t|$ \\
 \midrule
Source-only  & 0.1120  & 0.4323 &0.3201 \\
PADA  & 0.0577  & 0.0870 & 0.0293\\
 \midrule
   \rowcolor{gray!25}
BUOT  & \textbf{0.0130} &  \textbf{0.0112} & \textbf{0.0018} \\
 \bottomrule
\end{tabular}
\end{table}

\begin{figure}[!t]
    \centering
    \subfloat[ $\bm{\tilde{\Gamma}}^1$ (I$\rightarrow$P) \label{fig:heat_CLEF_i2p_1}]{\includegraphics[width=0.32\linewidth,height=80pt,trim=57 82 52 55,clip]{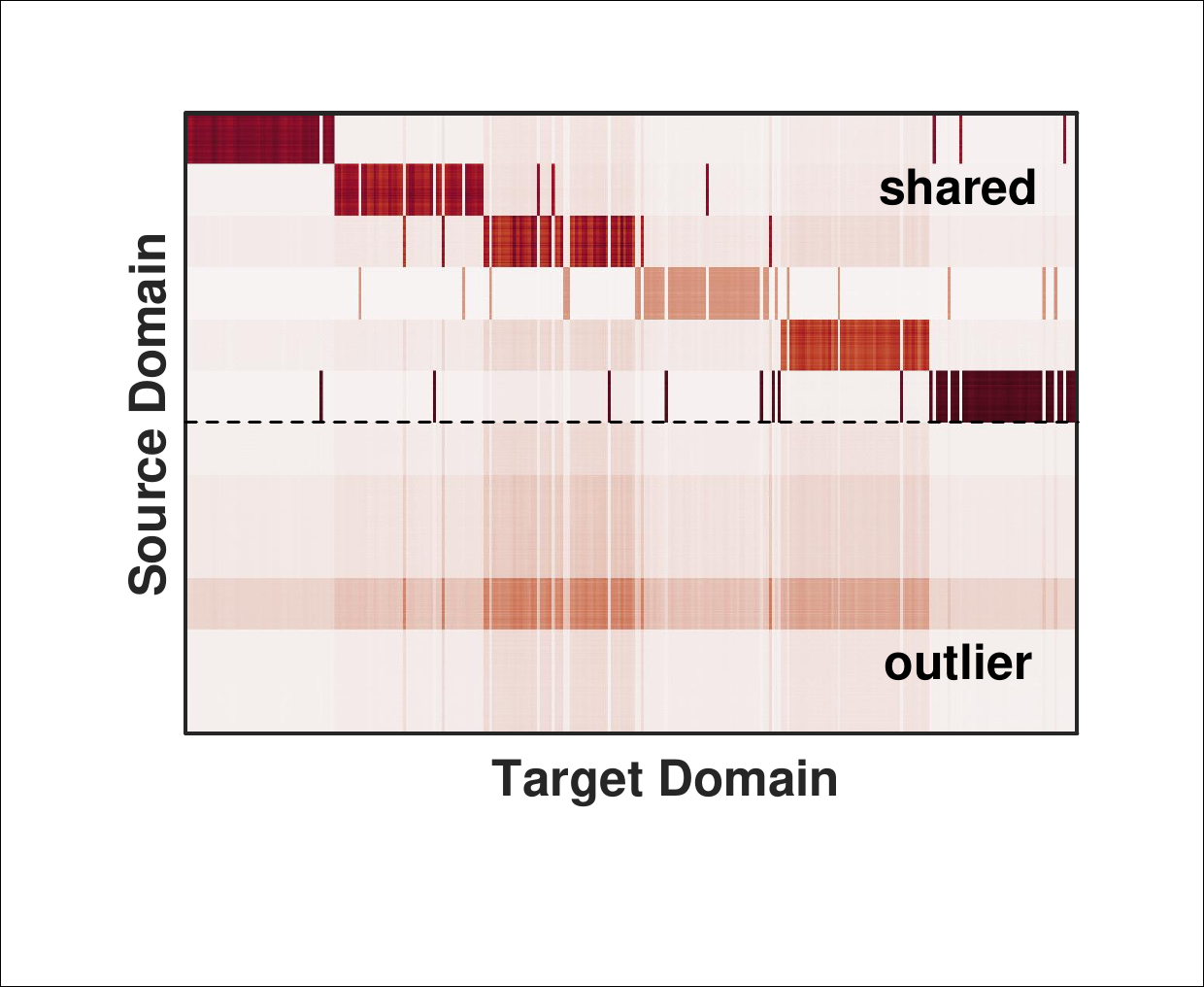}}
    \subfloat[ $\bm{\tilde{\Gamma}}^2$ (I$\rightarrow$P) \label{fig:feat_CLEF_i2p_2}]{\includegraphics[width=0.32\linewidth,height=80pt,trim=62 104 50 98,clip]{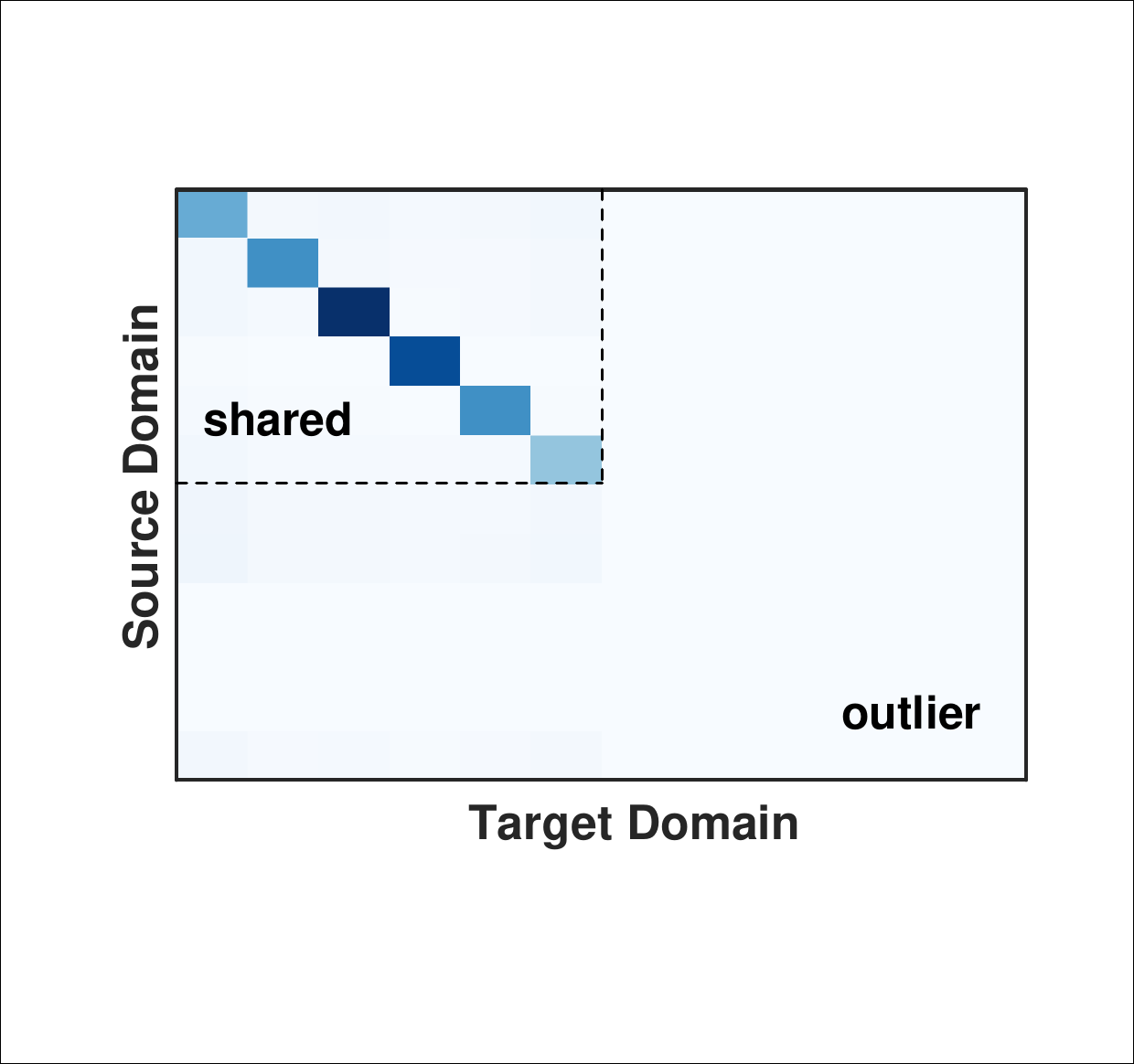}}
    \hfill
    \\
    \subfloat[$\bm{\tilde{\Gamma}}^1$ (I$\rightarrow$C) \label{fig:heat_CLEF_i2c_1}]{\includegraphics[width=0.32\linewidth,height=80pt,trim=61 103 53 84,clip]{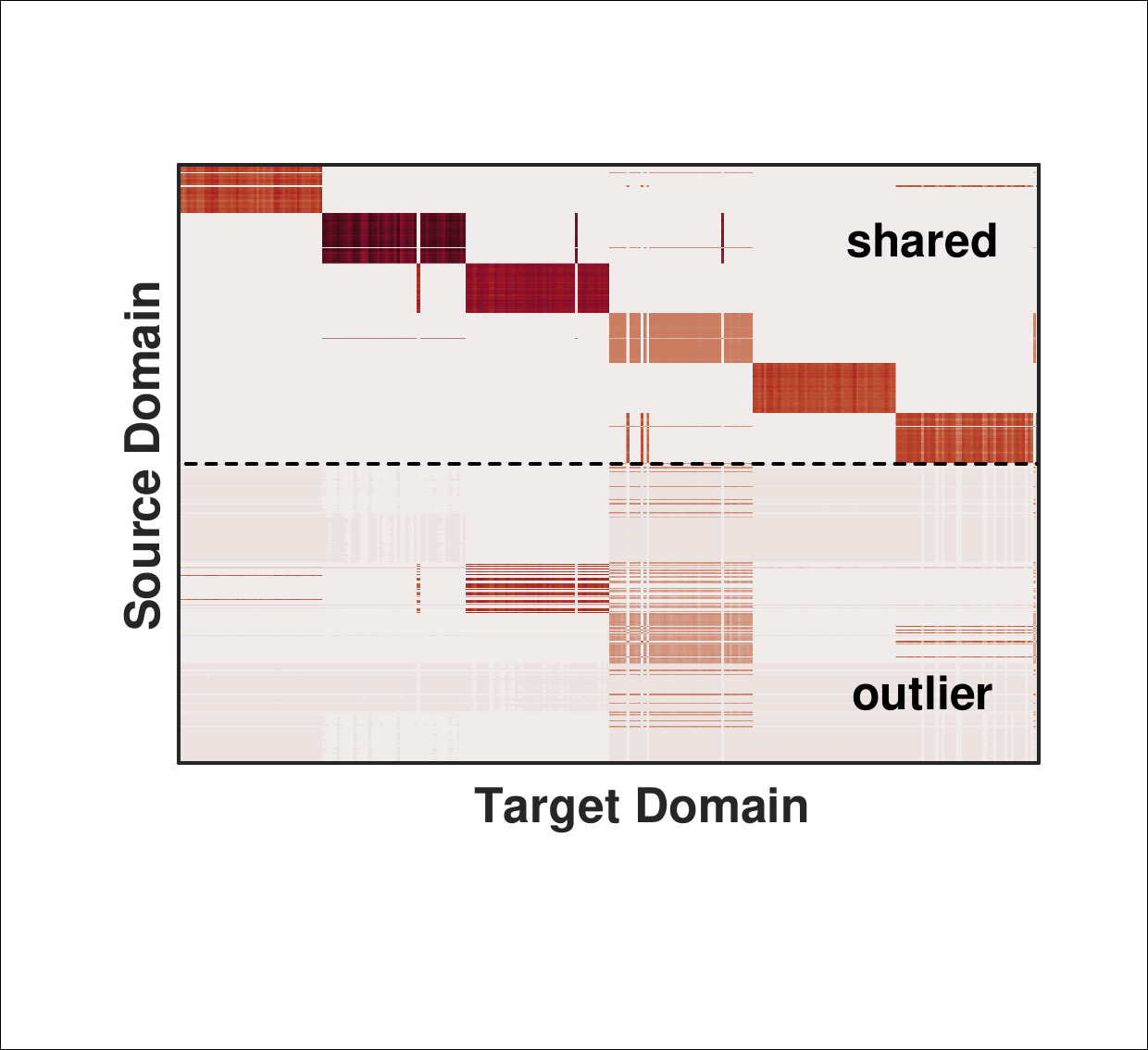}}
    \subfloat[$\bm{\tilde{\Gamma}}^2$ (I$\rightarrow$C) \label{fig:heat_CLEF_i2c_2}]{\includegraphics[width=0.32\linewidth,height=80pt,trim=42 112 55 45,clip]{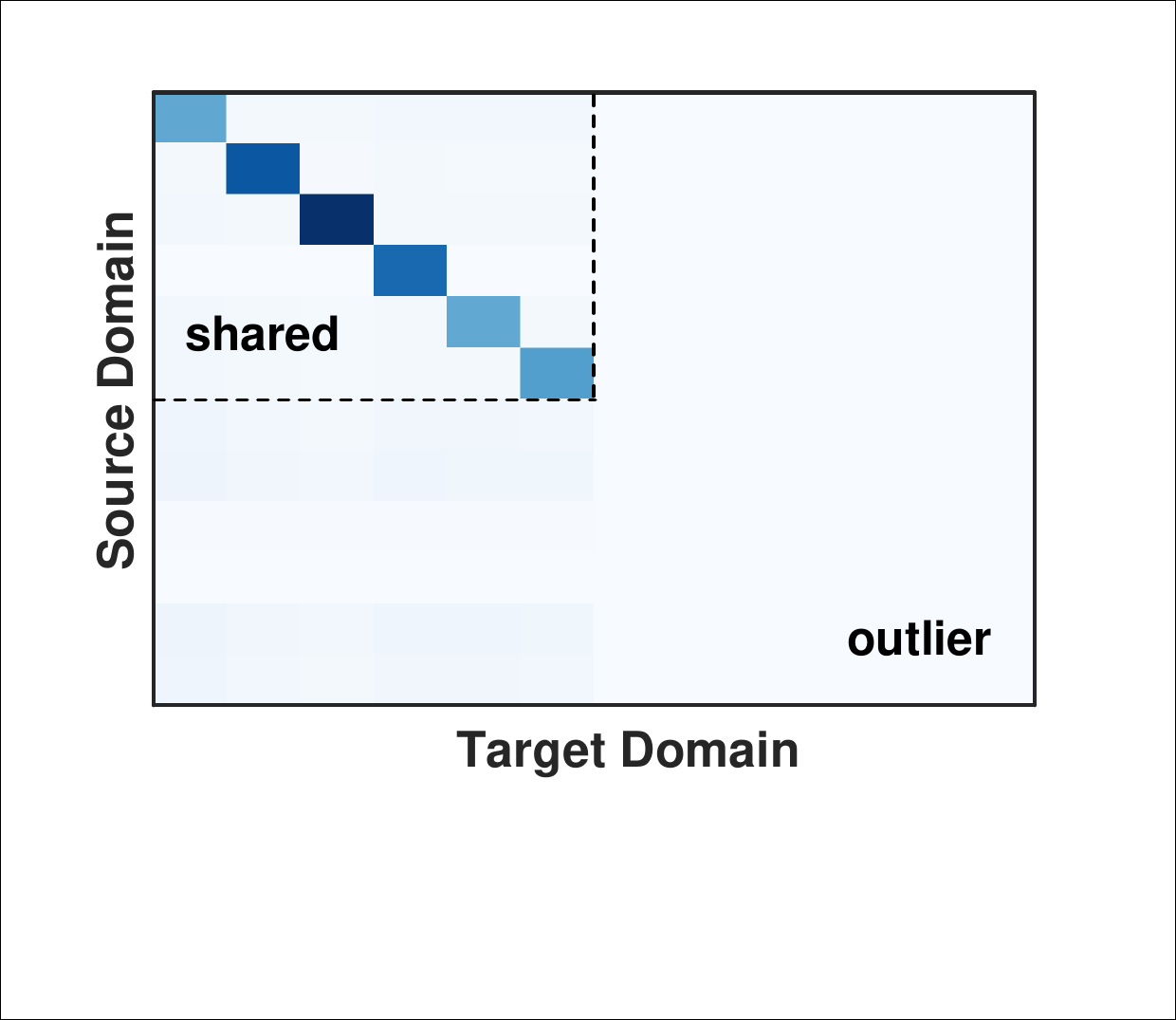}}
    \\
    \caption{The heatmaps of recovered sample-level transport $\bm{\tilde{\Gamma}}^1$ and recovered class-level transport $\bm{\tilde{\Gamma}}^2$, respectively. Results on ImageCLEF I$\rightarrow$P and ImageCLEF I$\rightarrow$C are presented. Best viewed in color.}
    \label{fig:heat}
\end{figure}
We present the weighted source domain risk, target domain risk, and model generalization error of BUOT in Tab.~\ref{tab:risk}. The weighted source domain risk and target domain risk respectively measure the model error on the source domain weighted by bi-level weights and on the target domain. A smaller generalization error indicates that the performance of the model on the source domain and target domain is more consistent, implying better generalization ability. BUOT achieves the smallest errors in both the weighted source domain and target domain compared to other methods, and it also has the best generalization ability. This demonstrates that our method is suitable for the PDA problem, as it not only aligns the shared classes between the source and target domains but also achieves correct classification.

\begin{figure*}[t]
    \centering
    \subfloat[SO \label{fig:cw_so}]{\includegraphics[width=0.43\linewidth,trim=25 45 95 70,clip]{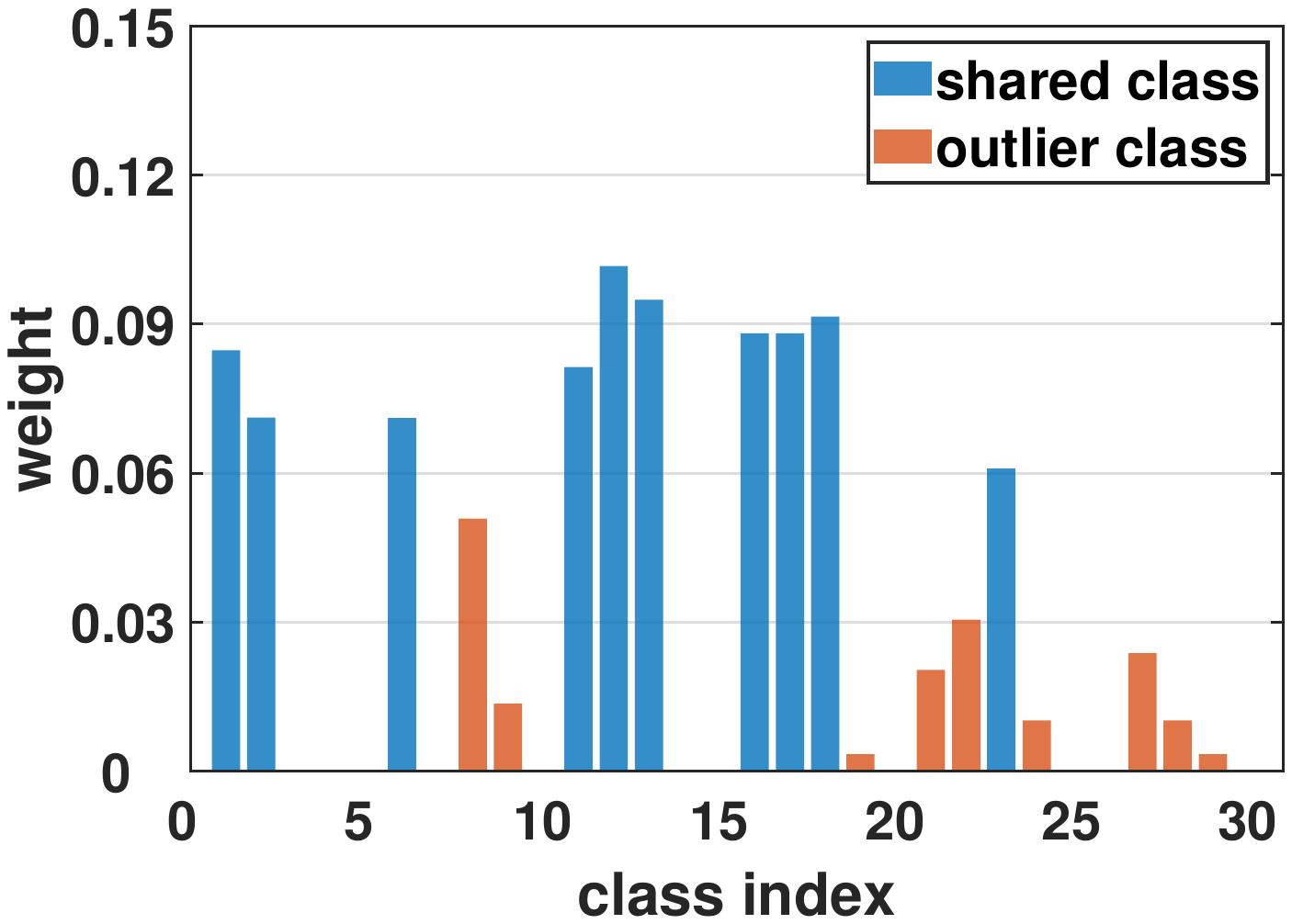}}
    \subfloat[PADA  \label{fig:cw_pa}]{\includegraphics[width=0.43\linewidth,trim=60 45 60 70,clip]{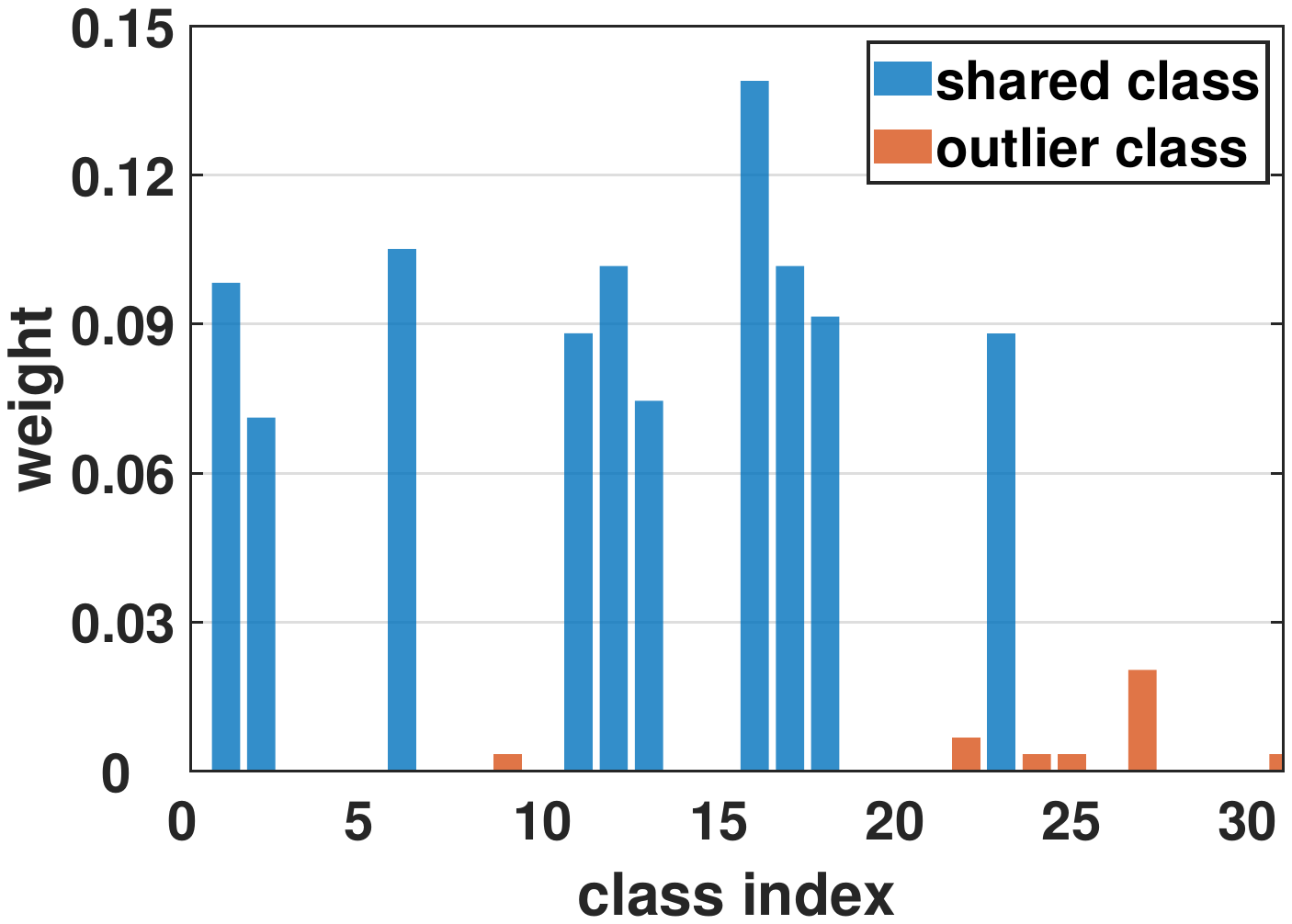}}
    \hfill\\
    \subfloat[BUOT  \label{fig:cw_bu}]{\includegraphics[width=0.43\linewidth,trim=35 45 85 70,clip]{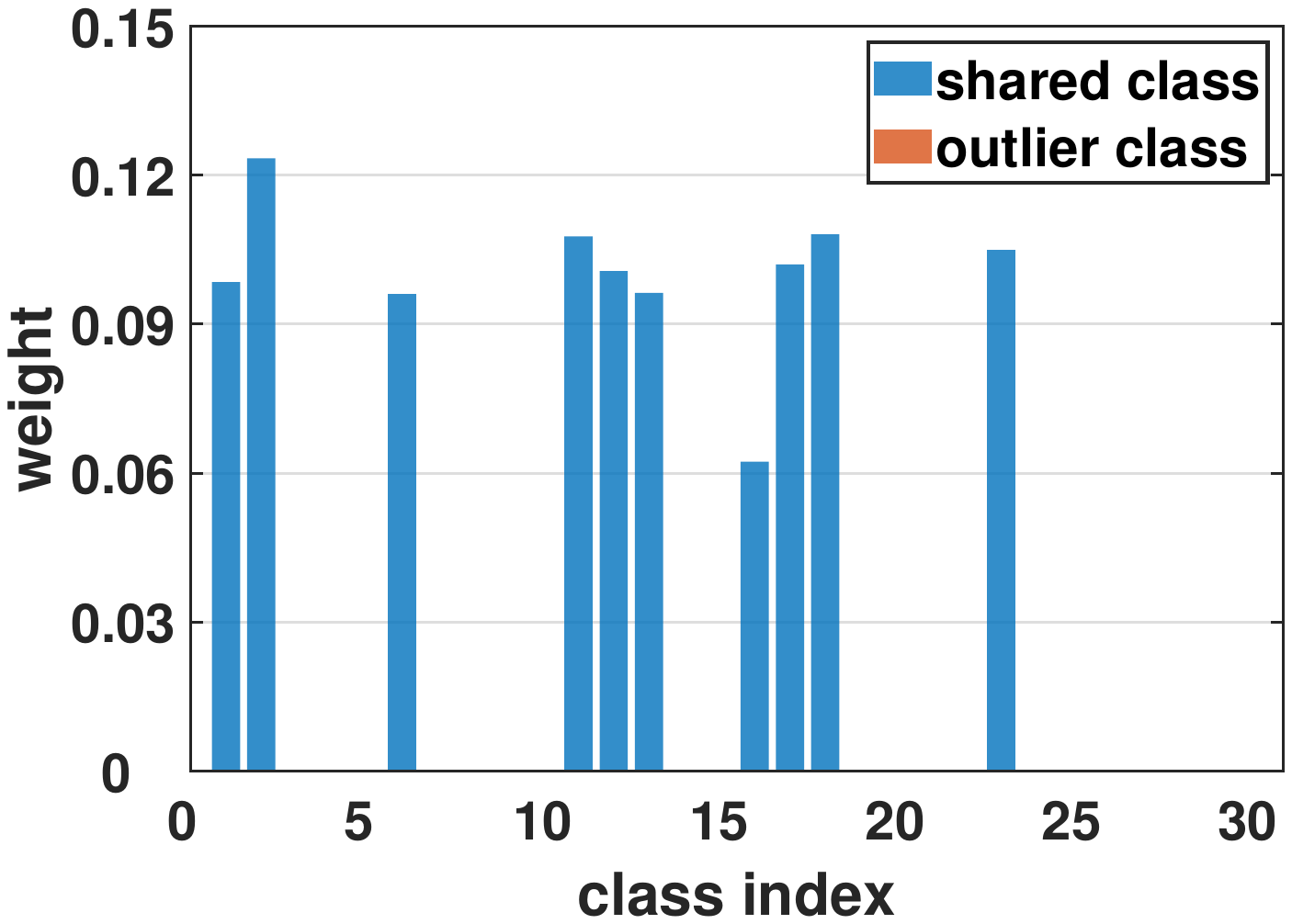}}
    \subfloat[Comparison \label{fig:2_class}]{\includegraphics[width=0.43\linewidth,trim=55 43 70 65,clip]{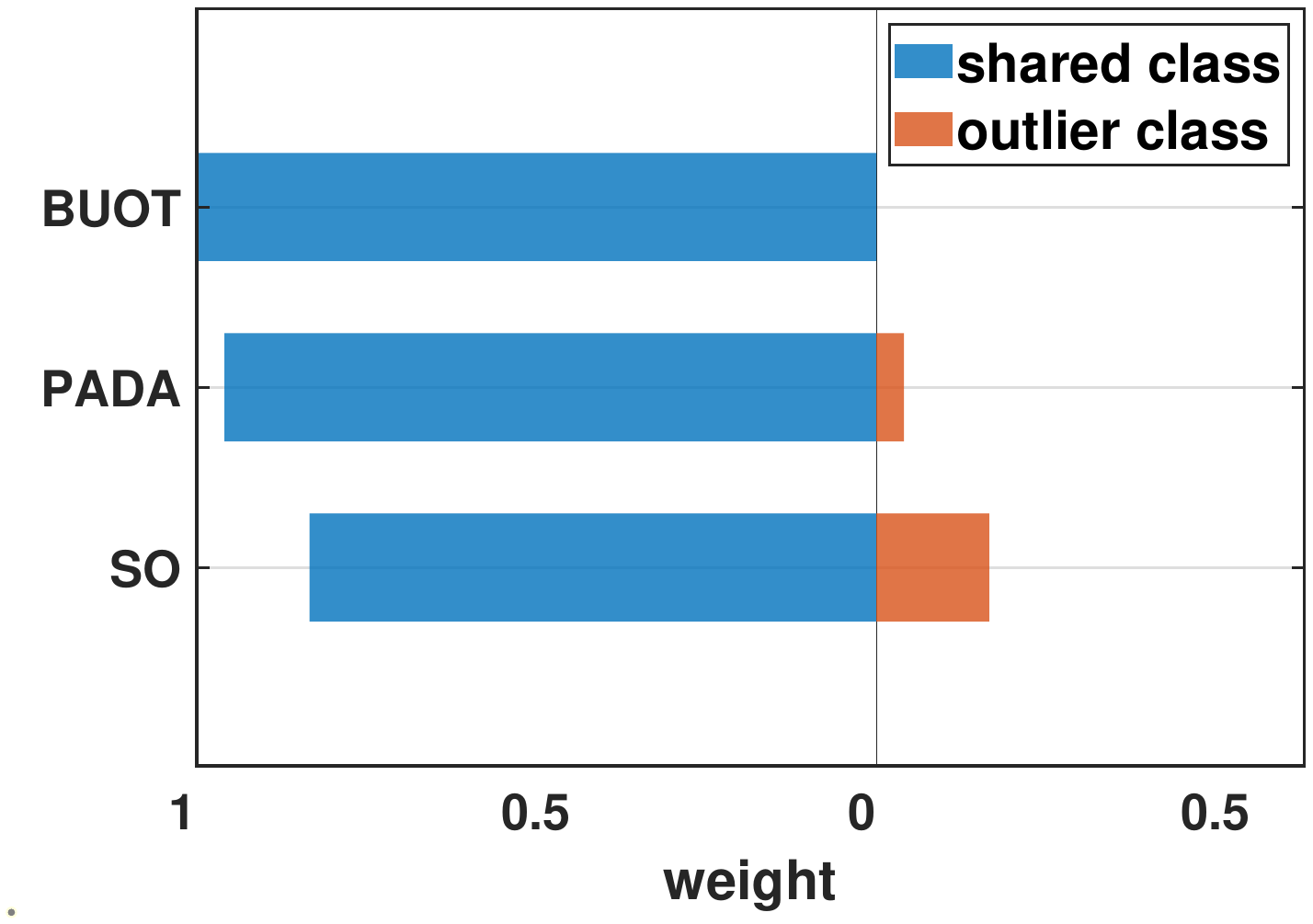}}
    \hfill
    \\
    \caption{(a)-(c): Histograms of class weights on Office-31 A$\rightarrow$W. The sums of weights is 1. (d): Comparison of identifying outlier classes. Best viewed in color.}
    \label{fig:cw}
\end{figure*}

\subsubsection{Cross-domain Structure Alignment}

\begin{figure}[t]
    \centering
    \subfloat[\scriptsize Office-31 A$\rightarrow$W \label{fig:class_CLEF}]{\includegraphics[width=0.42\linewidth,trim=52 45 72 70,clip]{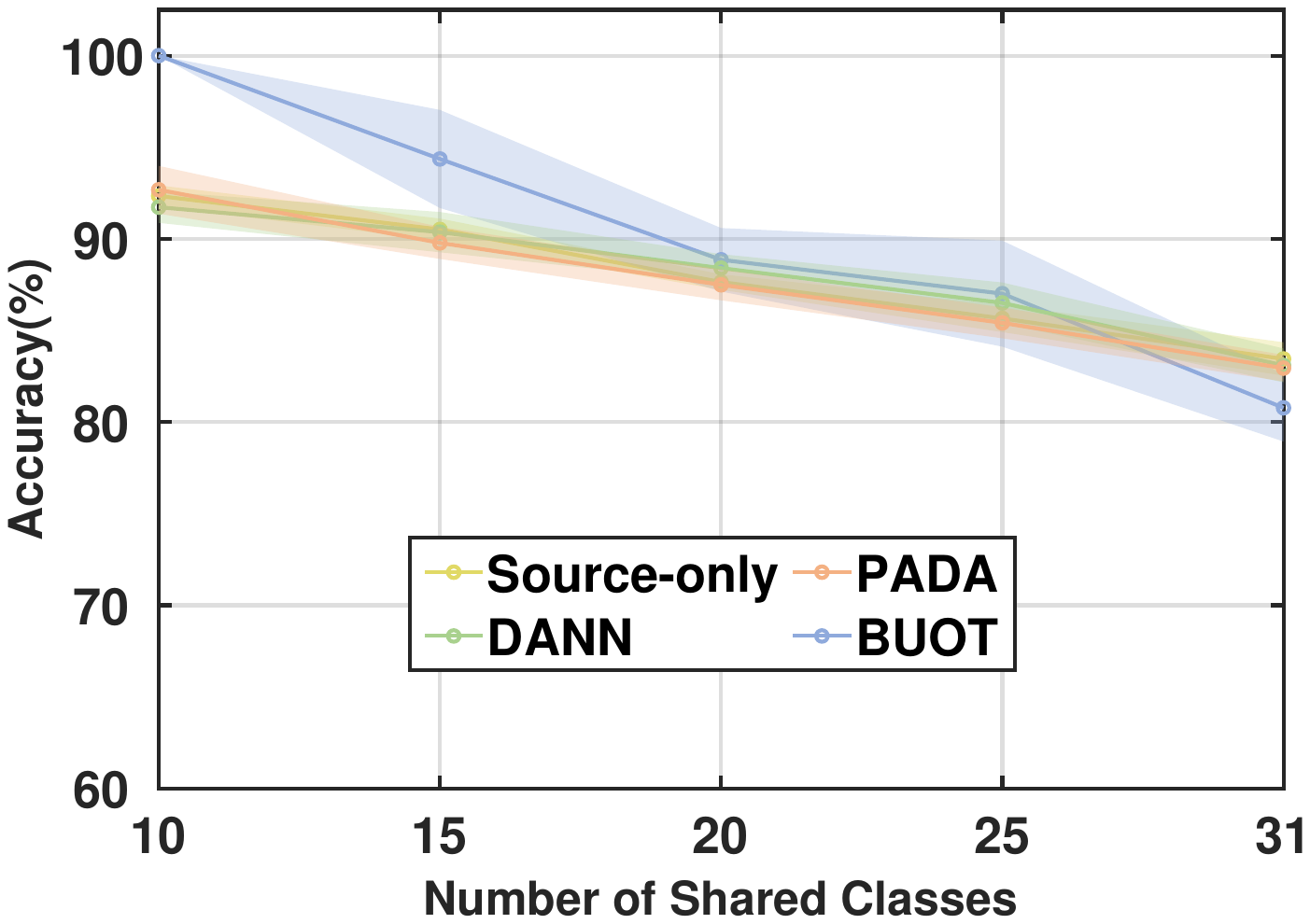}}
    \hfill
    \subfloat[\scriptsize Office-31 W$\rightarrow$A \label{fig:class_31}]{\includegraphics[width=0.42\linewidth,trim=52 45 72 70,clip]{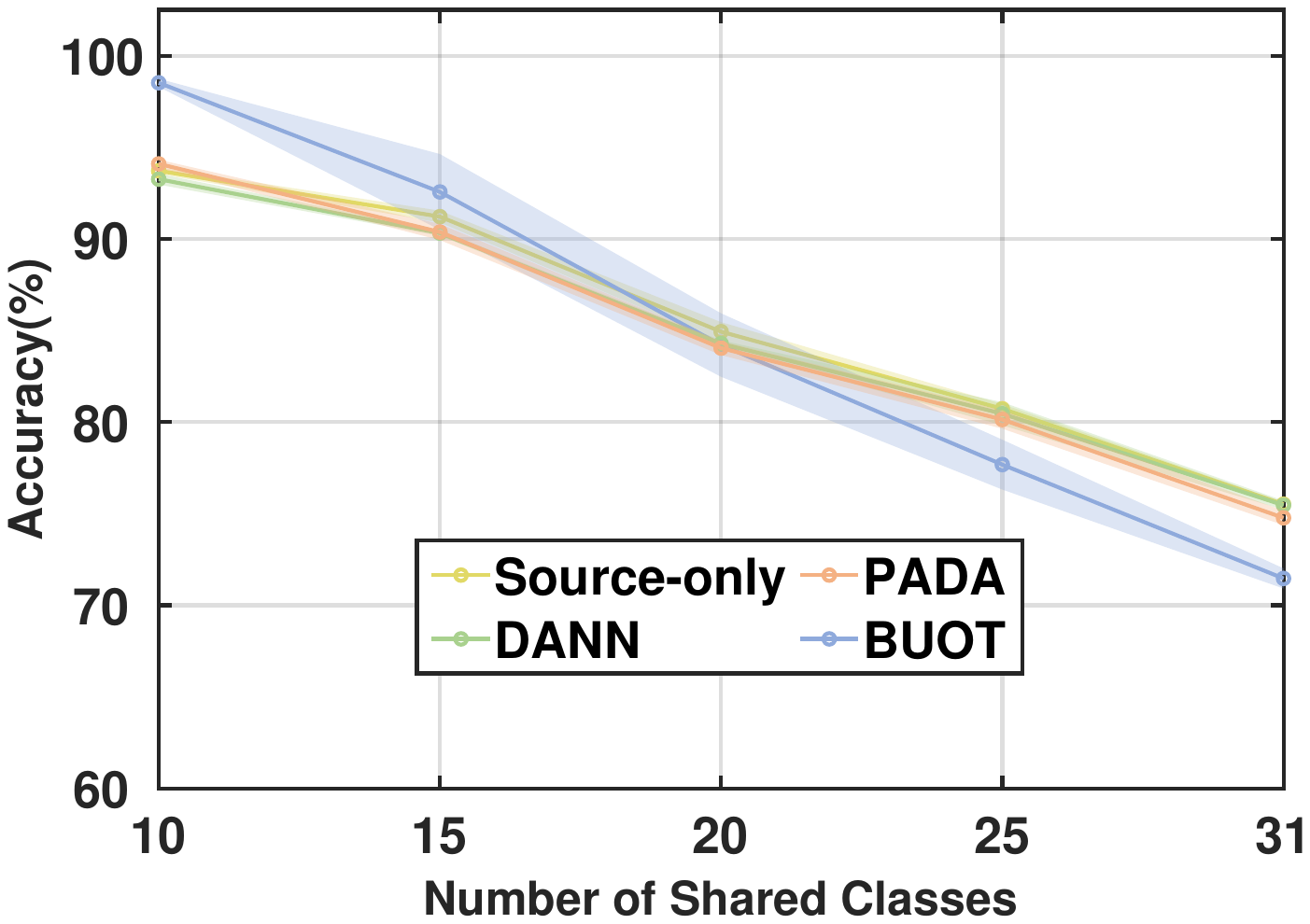}}
    \hfill
    \\
    \caption{Accuracy with respect to different numbers of target classes on Office-31 A$\rightarrow$W and W$\rightarrow$A. Best viewed in color.}
    \label{fig:target_class}
\end{figure}

In order to observe the cross-domain structure of bi-level optimal transport, we visualize the covered optimal transport plan $\bm{\tilde{\Gamma}}^1$ and $\bm{\tilde{\Gamma}}^2$ by heatmaps in Fig.~\ref{fig:heat}. As shown in Fig.~\subref*{fig:heat_CLEF_i2p_1} and Fig.~\subref*{fig:heat_CLEF_i2c_1}, for the ImageCLEF dataset, the first six classes are designated as shared classes, and the transport primarily occurs between samples within these shared classes. The clear block diagonal structure in the upper part indicates that the model can correctly perform intra-class cross-domain transport. Similarly, it is observed that in Fig.~\subref*{fig:feat_CLEF_i2p_2} and Fig.~\subref*{fig:heat_CLEF_i2c_2}, the transport almost exclusively occurs within the cross-domain shared classes. Through this experiment, we demonstrate that the BUOT model not only facilitates cross-domain alignment but also identifies shared classes and outlier classes in the source domain. This result demonstrates the effectiveness of BUOT in cross-domain transport tasks.

\subsubsection{Learning of class weights}
We compare the weights learned by BUOT with those learned by Source-only (SO) and PADA, where the weights in Source-only are the average of target predictions. SO relies solely on the cross-entropy loss from the source domain. To facilitate the comparison, we normalize the weights so that the sum of all weights is 1. Fig.~\subref*{fig:cw_so}-\subref*{fig:cw_bu} show the weights learned by SO, PADA, and BUOT on the A$\rightarrow$W task, respectively. In Fig.~\subref*{fig:2_class}, we divide the class weights into shared class weights and outlier class weights. The horizontal axis in Fig.~\subref*{fig:2_class} represents the values of the weights, with the left side indicating the weights for the shared class and the right side indicating the weights for the outlier class. It can be seen that BUOT can learn the correct class weights, with almost no weights assigned to outlier classes, whereas SO and PADA incorrectly identified some outlier classes as shared classes. This demonstrates that BUOT can indeed reduce the deviation of class weights from the true weights.

\subsubsection{Accuracy With Respect to Different Numbers of Target Classes}

To demonstrate the generalizability of our method for changes in the number of target classes, we show in Fig.~\ref{fig:target_class} the change in accuracy of BUOT when the number of target classes changes. To demonstrate the performance of our model under different PDA settings, we compared it with the classic UDA model DANN, the classic PDA model PDDA, and SO. We conducted 5 random experiments, with the lines representing the mean values and the shaded areas indicating the 95\% confidence intervals, which is consistent across subsequent experiments. As can be seen, as the number of target classes decreases, the accuracy of BUOT is constantly increasing and is almost always higher than the other methods. This demonstrates the effectiveness of our approach to the PDA problem.

\begin{figure*}[t]
    \centering
    \subfloat[ImageCLEF I$\rightarrow$P \label{fig:con_CLEF}]{\includegraphics[width=0.43\linewidth,trim=50 45 40 70,clip]{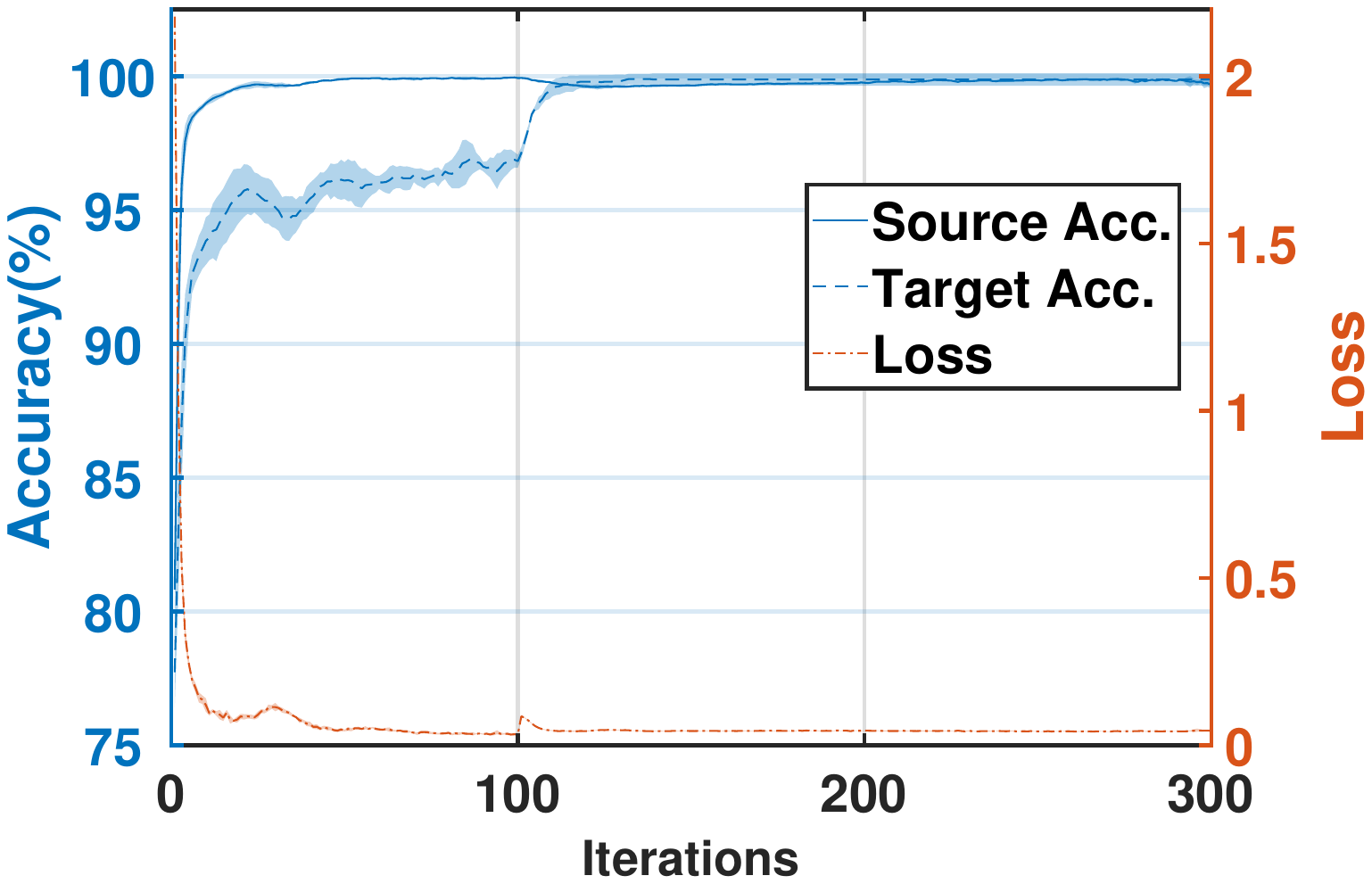}}
    \hfill
    \subfloat[ Office-31 A$\rightarrow$W \label{fig:con_31}]{\includegraphics[width=0.43\linewidth,trim=50 45 40 70,clip]{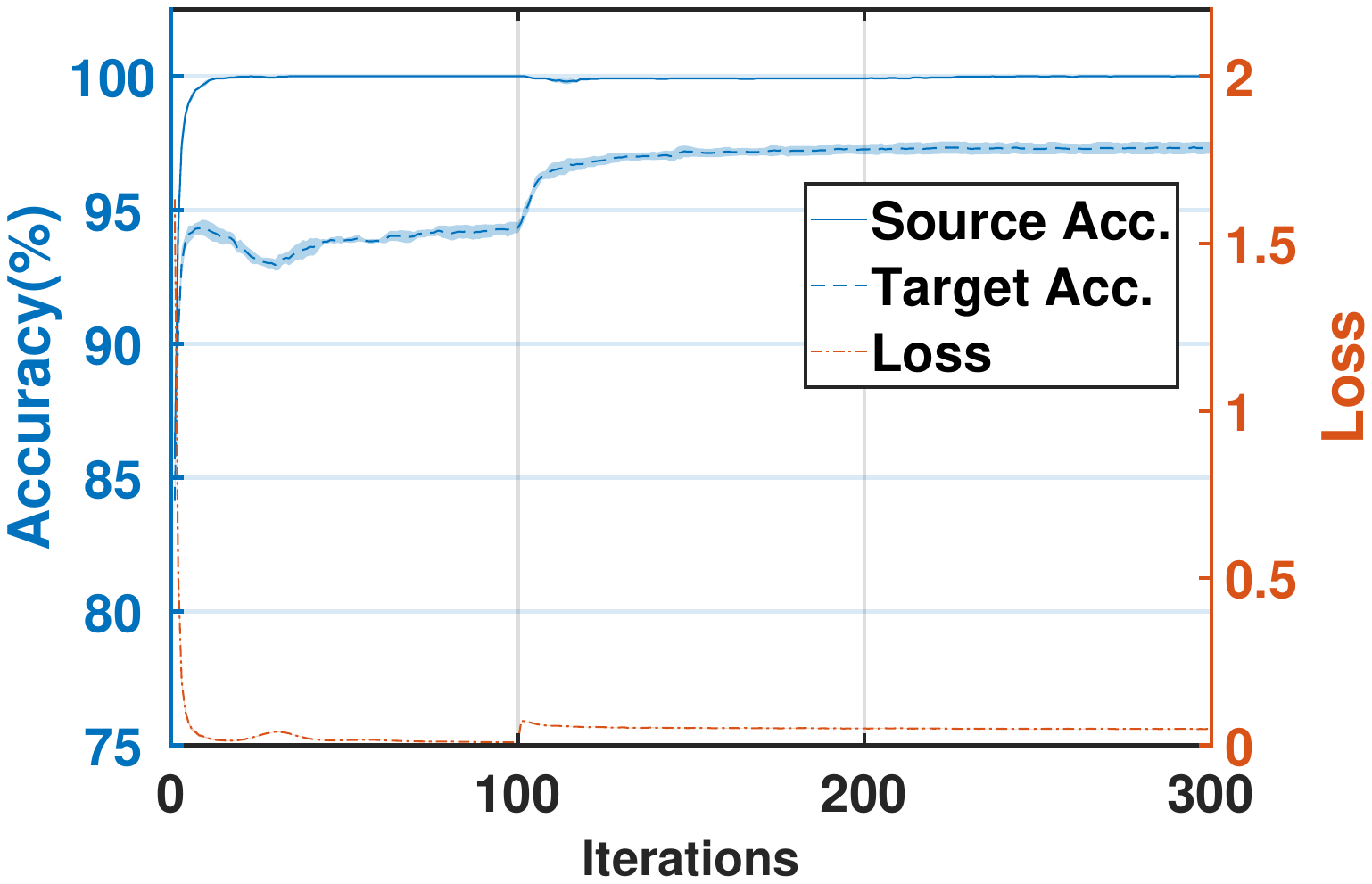}}
    \hfill\\ 
    \subfloat[ ImageCLEF I$\rightarrow$P \label{fig:hyper_CLEF}]{\includegraphics[width=0.43\linewidth,trim=50 45 65 70,clip]{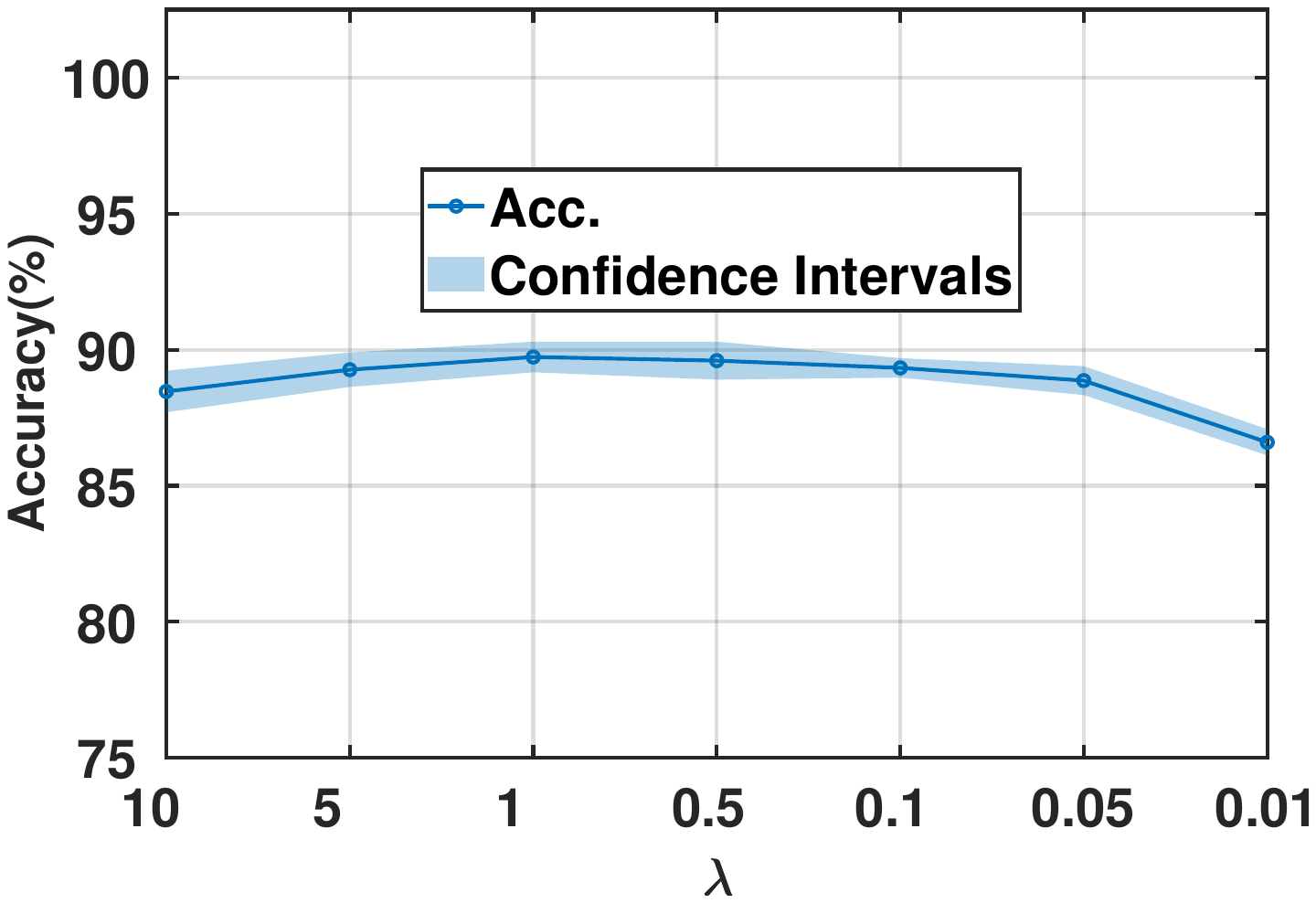}}
    \hfill
    \subfloat[Office-31 A$\rightarrow$W \label{fig:hyper_31}]{\includegraphics[width=0.43\linewidth,trim=50 45 65 70,clip]{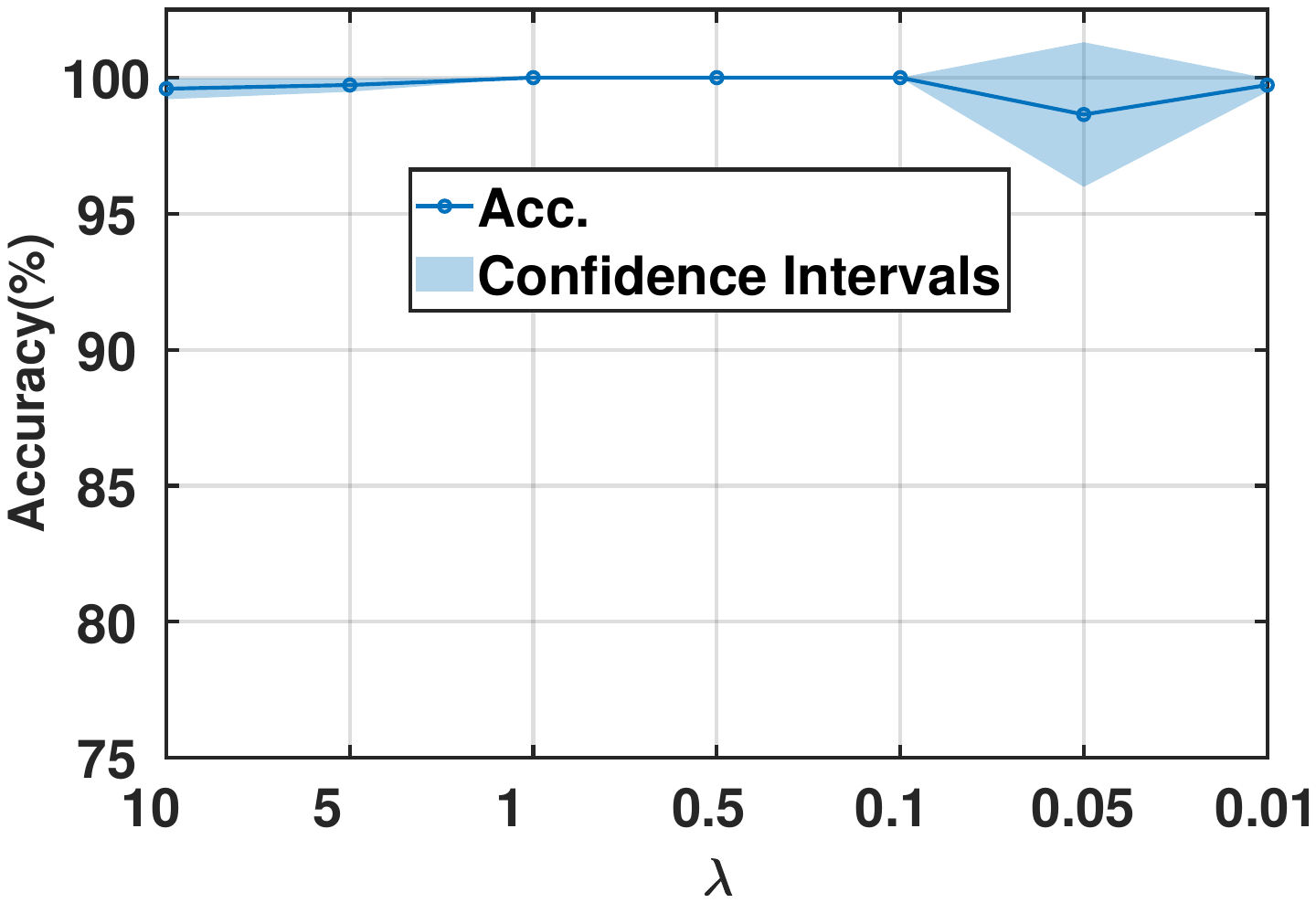}}
    \hfill
    \\
    \caption{The convergence analysis and the effect of hyper-parameters on ImageCLEF I$\rightarrow$P and Office-31 A$\rightarrow$W. Best viewed in color.}
    \label{fig:covergence}
\end{figure*}

\subsubsection{Convergence Analysis}

Fig.~\subref*{fig:con_CLEF}-\subref*{fig:con_31} shows the source domain classification accuracy, target domain classification accuracy, and value of BUOT loss for ImageCLEF I$\rightarrow$P and Office-31 A$\rightarrow$W, respectively. The accuracy is represented in blue, and the loss is represented in orange. Since the prediction usually induces uncertainty, we adopt a warm-up training strategy in the first 100 iterations to mitigate its negative impacts. In Fig.~\subref*{fig:con_CLEF}, it can be seen that in the warm-up stage, the target domain accuracy gradually increases to around 80\% and then tends to be stable, so the reliability of the prediction is guaranteed. Then the bi-level transprt makes the results further improved to 90\%, which indicates that BUOT can provide the label information of the target domain. In Fig.~\subref*{fig:con_31}, in the Office-31 dataset, the accuracy of the prediction in the warm-up phase reaches more than 90\%, and then the accuracy of the target domain is also significantly improved. As can be seen from Fig.~\ref{fig:covergence}, the BUOT loss consistently decreases and converges with a very small confidence interval, indicating that the model training is highly stable.

\subsubsection{Hyper-parameter Sensitivity}
We evaluate the hyper-parameters $\lambda$ on ImageCLEF I$\rightarrow$P and Office-31 A$\rightarrow$W. The results are shown in Fig.~\subref*{fig:hyper_CLEF}-\subref*{fig:hyper_31}. The value of $\lambda$ ranges from 10 to 0.01. The results demonstrate that the model is robust to changes in hyperparameters. When $\lambda=0.05$ in Office-31, the confidence interval becomes relatively wide, potentially due to instability caused by the reduced BUOT loss. Therefore $\lambda$ should be set to a larger value to mitigate this problem. Specifically, in ImageCLEF, the accuracy reaches its highest point when $\lambda=1$, while in Office-31, the accuracy reaches 100\% when $\lambda \in (0.1, 1) $.

\begin{figure*}[!t]
    \centering
    \subfloat[ SO (domain) \label{fig:tsne_CLEF_So_d}]{\includegraphics[width=0.24\linewidth,trim=70 270 70 200,clip]{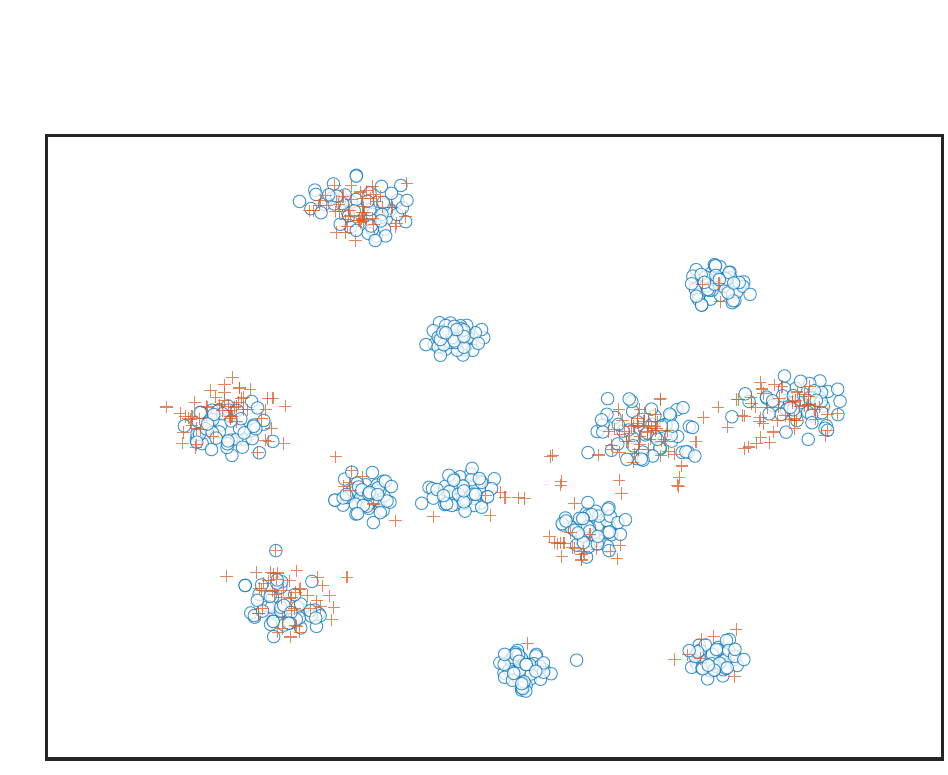}}
    \hfill
    \subfloat[ SO (class) \label{fig:tsne_CLEF_So_c}]{\includegraphics[width=0.24\linewidth,trim=70 270 70 200,clip]{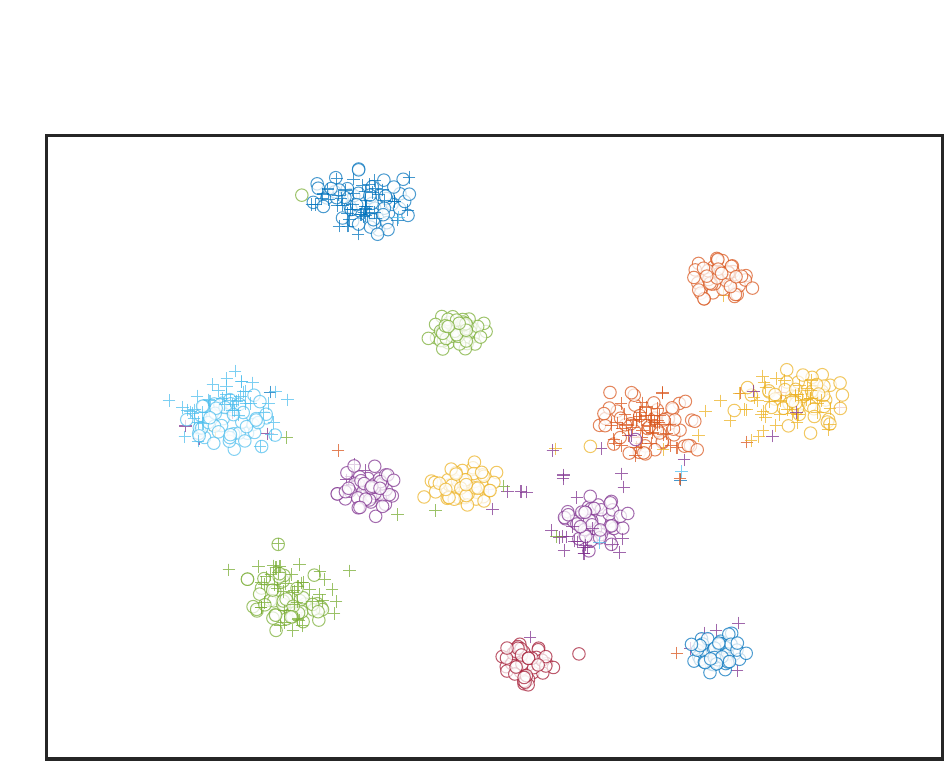}}
    \hfill
    \subfloat[ BUOT (domain) \label{fig:tsne_CLEF_BUOT_d}]{\includegraphics[width=0.24\linewidth,trim=70 270 70 200,clip]{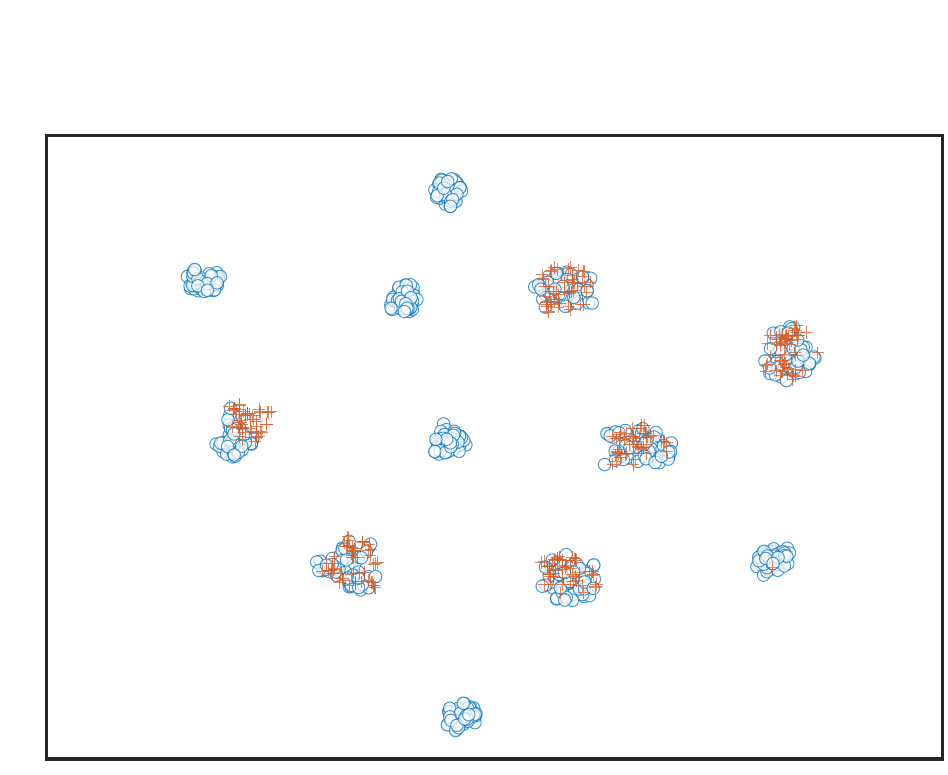}}
    \hfill
    \subfloat[BUOT (class) \label{fig:tsne_CLEF_BUOT_c}]{\includegraphics[width=0.24\linewidth,trim=70 270 70 200,clip]{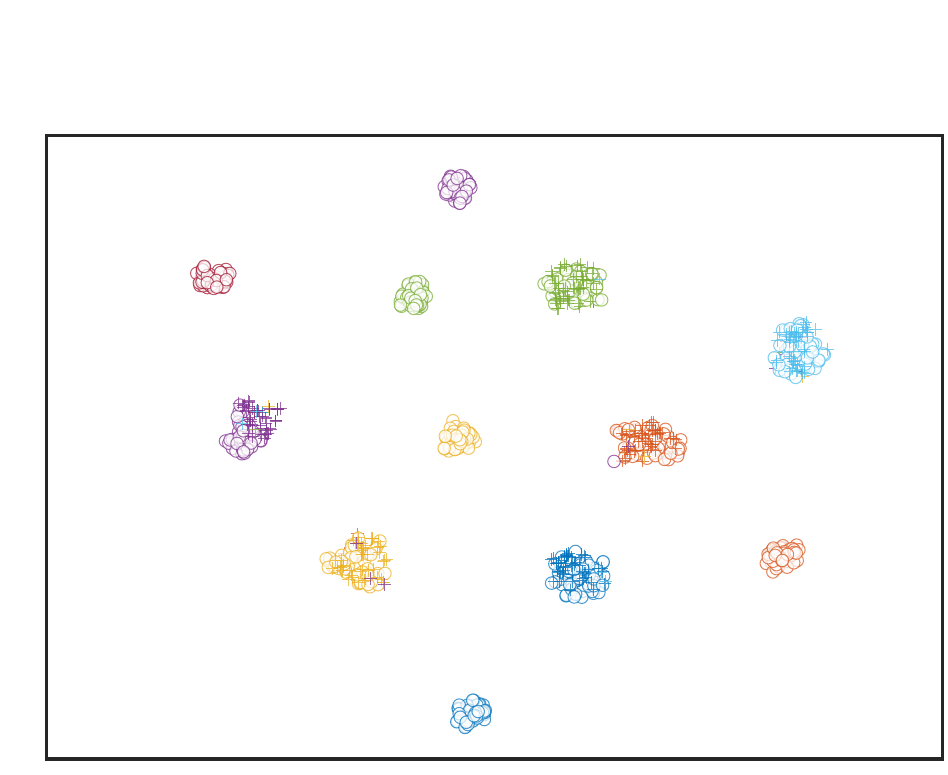}}
    \\
    \caption{t-SNE visualization of Source-only and BUOT representations on ImageCLEF I$\rightarrow$P, different colors indicate different classes.`$\circ$'$\colon$source domain, `+'$\colon$target domain. Best viewed in color.}
    \label{fig:tSNE}
\end{figure*}
\subsubsection{Feature Visualization}

To clearly show the role of the BUOT model in dealing with the PDA problem, we use t-SNE~\cite{van2008visualizing} to visualize the representations learned by SO and BUOT in Fig.~\ref{fig:tSNE}. The t-SNE visualization in Fig.~\subref*{fig:tsne_CLEF_So_d} reveals that SO exhibits limited recognition capabilities for outlier classes. In the ImageCLEF dataset, the target domain only has 6 classes, however, the representations learned through SO in the target domain obviously have more than 6 classes. From Fig.~\subref*{fig:tsne_CLEF_So_c}, we can see that the representations learned by SO are not discriminative, with ambiguous inter-class margins. Additionally, many samples of different classes are overlapped, indicating a mismatch between the two domains. From Fig.~\subref*{fig:tsne_CLEF_BUOT_d} and Fig.~\subref*{fig:tsne_CLEF_BUOT_c}, we can see that the representations learned by BUOT not only correctly recognize the outlier classes, but also achieve intra-domain separation between different classes and cross-domain alignment within the same classes. The results above validate that the BUOT model indeed ensures a better representation space with discriminability.

\section{Conclusion}
\label{chap:5}
In this paper, we aim to deal with the label space inconsistency problem between source and target domains. How identifying the source domain outlier classes is very important for the PDA problem. We propose the BUOT method to simultaneously learn the sample-level and class-level transport between the source and target domains. Further, we can recover the sample-wise and class-wise relations and then obtain the bi-level weights to recognize the outlier classes. To improve the discriminability of the model, we propose a novel label-aware transport cost based on label indices. Extensive experiments demonstrate the effectiveness of BUOT.

How to use the idea of BUOT to deal with open-set domain adaptation task is our future work.

\section{Acknowledgment}

This work is supported in part by National Natural Science Foundation of China (62376291), Science and Technology Program of Guangzhou (2024A04J6413), Sun Yat-sen University (24xkjc013), and in part by the Hong Kong Innovation and Technology Commission (ITC) (InnoHK Project CIMDA) and the Institute of Digital Medicine of City University of Hong Kong (Project 9229503).

\small
\bibliographystyle{elsarticle-num}  
\bibliography{ref-buot}

\begin{thebibliography}{10}
\expandafter\ifx\csname url\endcsname\relax
  \def\url#1{\texttt{#1}}\fi
\expandafter\ifx\csname urlprefix\endcsname\relax\def\urlprefix{URL }\fi
\expandafter\ifx\csname href\endcsname\relax
  \def\href#1#2{#2} \def\path#1{#1}\fi

\bibitem{kerdoncuff2021metric}
T.~Kerdoncuff, R.~Emonet, M.~Sebban, Metric learning in optimal transport for
  domain adaptation, in: CVPR, 2021, pp. 2162--2168.

\bibitem{XU2023109787}
X.-L. Xu, G.-X. Xu, C.-X. Ren, D.-Q. Dai, H.~Yan, Conditional independence
  induced unsupervised domain adaptation, PR 143 (2023) 109787.

\bibitem{tzeng2014deep}
E.~Tzeng, J.~Hoffman, N.~Zhang, K.~Saenko, T.~Darrell, Deep domain confusion:
  Maximizing for domain invariance, arXiv preprint arXiv:1412.3474.

\bibitem{kang2019contrastive}
G.~Kang, L.~Jiang, Y.~Yang, A.~G. Hauptmann, Contrastive adaptation network for
  unsupervised domain adaptation, in: CVPR, 2019, pp. 4893--4902.

\bibitem{xia2023maximum}
H.~Xia, T.~Jing, Z.~Ding, Maximum structural generation discrepancy for
  unsupervised domain adaptation, IEEE TPAMI 45~(3) (2023) 3434--3445.

\bibitem{ganin2016domain}
Y.~Ganin, E.~Ustinova, H.~Ajakan, P.~Germain, H.~Larochelle, F.~Laviolette,
  M.~Marchand, V.~Lempitsky, Domain-adversarial training of neural networks,
  JMLR 17~(1) (2016) 2096--2030.

\bibitem{dhouib2020margin}
S.~Dhouib, I.~Redko, C.~Lartizien, Margin-aware adversarial domain adaptation
  with optimal transport, in: ICML, 2020, pp. 2514--2524.

\bibitem{long2018transferable}
M.~Long, Y.~Cao, Z.~Cao, J.~Wang, M.~I. Jordan, Transferable representation
  learning with deep adaptation networks, IEEE TPAMI 41~(12) (2018) 3071--3085.

\bibitem{thota2021contrastive}
M.~Thota, G.~Leontidis, Contrastive domain adaptation, in: CVPR, 2021, pp.
  2209--2218.

\bibitem{zhang2020optimal}
Z.~Zhang, M.~Wang, A.~Nehorai, Optimal transport in reproducing kernel hilbert
  spaces: Theory and applications, IEEE TPAMI 42~(7) (2020) 1741--1754.

\bibitem{chen2018re}
Q.~Chen, Y.~Liu, Z.~Wang, I.~Wassell, K.~Chetty, Re-weighted adversarial
  adaptation network for unsupervised domain adaptation, in: CVPR, 2018, pp.
  7976--7985.

\bibitem{shen2018wasserstein}
J.~Shen, Y.~Qu, W.~Zhang, Y.~Yu, Wasserstein distance guided representation
  learning for domain adaptation, in: AAAI, Vol.~32, 2018.

\bibitem{russakovsky2015imagenet}
O.~Russakovsky, J.~Deng, H.~Su, J.~Krause, S.~Satheesh, S.~Ma, Z.~Huang,
  A.~Karpathy, A.~Khosla, M.~Bernstein, et~al., Imagenet large scale visual
  recognition challenge, IJCV 115 (2015) 211--252.

\bibitem{griffin2007caltech}
G.~Griffin, A.~Holub, P.~Perona, et~al., Caltech-256 object category dataset,
  Tech. rep., Technical Report 7694, California Institute of Technology
  Pasadena (2007).

\bibitem{cao2018partial}
Z.~Cao, M.~Long, J.~Wang, M.~I. Jordan, Partial transfer learning with
  selective adversarial networks, in: CVPR, 2018, pp. 2724--2732.

\bibitem{ren2020learning}
C.-X. Ren, P.~Ge, P.~Yang, S.~Yan, Learning target-domain-specific classifier
  for partial domain adaptation, IEEE TNNLS 32~(5) (2020) 1989--2001.

\bibitem{yang2023contrastive}
C.~Yang, Y.-M. Cheung, J.~Ding, K.~C. Tan, B.~Xue, M.~Zhang, Contrastive
  learning assisted-alignment for partial domain adaptation, IEEE TNNLS 34~(10)
  (2023) 7621--7634.

\bibitem{cao2018partialada}
Z.~Cao, L.~Ma, M.~Long, J.~Wang, Partial adversarial domain adaptation, in:
  ECCV, 2018, pp. 135--150.

\bibitem{cao2023big}
Z.~Cao, K.~You, Z.~Zhang, J.~Wang, M.~Long, From big to small: Adaptive
  learning to partial-set domains, IEEE TPAMI 45~(2) (2023) 1766--1780.

\bibitem{cao2019learning}
Z.~Cao, K.~You, M.~Long, J.~Wang, Q.~Yang, Learning to transfer examples for
  partial domain adaptation, in: CVPR, 2019, pp. 2985--2994.

\bibitem{luo2021conditional}
Y.~W. Luo, C.~X. Ren, Conditional bures metric for domain adaptation, in: CVPR,
  2021, pp. 13989--13998.

\bibitem{Wang2024pp}
Y.~Wang, C.-X. Ren, Y.-M. Zhai, Y.-W. Luo, H.~Yan, Probability-polarized
  optimal transport for unsupervised domain adaptation, in: AAAI, Vol.~38,
  2024, pp. 15653--15661.

\bibitem{qian2022joint}
J.~Qian, W.~K. Wong, H.~Zhang, J.~Xie, J.~Yang, Joint optimal transport with
  convex regularization for robust image classification, IEEE TCYB 52~(3)
  (2022) 1553--1564.

\bibitem{courty2017joint}
N.~Courty, R.~Flamary, A.~Habrard, A.~Rakotomamonjy, Joint distribution optimal
  transportation for domain adaptation, in: NeurIPS, Vol.~30, 2017.

\bibitem{li2020enhanced}
M.~Li, Y.~Zhai, Y.~Luo, P.~Ge, C.~Ren, Enhanced transport distance for
  unsupervised domain adaptation, in: CVPR, 2020, pp. 13936--13944.

\bibitem{chizat2018unbalanced}
L.~Chizat, G.~Peyr{\'e}, B.~Schmitzer, F.-X. Vialard, Unbalanced optimal
  transport: Dynamic and kantorovich formulations, Journal of Functional
  Analysis 274~(11) (2018) 3090--3123.

\bibitem{chizat2018scaling}
L.~Chizat, G.~Peyr{\'e}, B.~Schmitzer, F.-X. Vialard, Scaling algorithms for
  unbalanced optimal transport problems, Mathematics of Computation 87~(314)
  (2018) 2563--2609.

\bibitem{zhang2018importance}
J.~Zhang, Z.~Ding, W.~Li, P.~Ogunbona, Importance weighted adversarial nets for
  partial domain adaptation, in: CVPR, 2018, pp. 8156--8164.

\bibitem{luo2022unsupervised}
Y.~W. Luo, C.~X. Ren, D.~Q. Dai, H.~Yan, Unsupervised domain adaptation via
  discriminative manifold propagation, IEEE TPAMI 44~(3) (2022) 1653--1669.

\bibitem{li2020deep}
S.~Li, C.~H. Liu, Q.~Lin, Q.~Wen, L.~Su, G.~Huang, Z.~Ding, Deep residual
  correction network for partial domain adaptation, IEEE TPAMI 43~(7) (2020)
  2329--2344.

\bibitem{lin2022adversarial}
K.-Y. Lin, J.~Zhou, Y.~Qiu, W.-S. Zheng, Adversarial partial domain adaptation
  by cycle inconsistency, in: ECCV, Springer, 2022, pp. 530--548.

\bibitem{gu2021adversarial}
X.~Gu, X.~Yu, J.~Sun, Z.~Xu, et~al., Adversarial reweighting for partial domain
  adaptation, NeurIPS 34 (2021) 14860--14872.

\bibitem{liang2020balanced}
J.~Liang, Y.~Wang, D.~Hu, R.~He, J.~Feng, A balanced and uncertainty-aware
  approach for partial domain adaptation, in: ECCV, 2020, pp. 123--140.

\bibitem{wu2023reinforced}
K.~Wu, M.~Wu, Z.~Chen, R.~Jin, W.~Cui, Z.~Cao, X.~Li, Reinforced adaptation
  network for partial domain adaptation, IEEE TCSVT 33~(5) (2023) 2370--2380.

\bibitem{chen2022domain}
J.~Chen, X.~Wu, L.~Duan, S.~Gao, Domain adversarial reinforcement learning for
  partial domain adaptation, IEEE TNNLS 33~(2) (2022) 539--553.

\bibitem{peyre2017computational}
G.~Peyr{\'e}, M.~Cuturi, et~al., Computational optimal transport, Center for
  Research in Economics and Statistics Working Papers~(2017-86).

\bibitem{kantorovich1942translocation}
L.~V. Kantorovich, On the translocation of masses, Proceedings of the USSR
  Academy of Sciences 37~(7-8) (1942) 227--229.

\bibitem{cuturi2013sinkhorn}
M.~Cuturi, Sinkhorn distances: Lightspeed computation of optimal transport, in:
  NeurIPS, Vol.~26, 2013.

\bibitem{memoli2011gromov}
F.~M{\'e}moli, Gromov--wasserstein distances and the metric approach to object
  matching, Foundations of computational mathematics 11 (2011) 417--487.

\bibitem{peyre2016gromov}
G.~Peyr{\'e}, M.~Cuturi, J.~Solomon, Gromov-wasserstein averaging of kernel and
  distance matrices, in: ICML, PMLR, 2016, pp. 2664--2672.

\bibitem{titouan2020co}
V.~Titouan, I.~Redko, R.~Flamary, N.~Courty, Co-optimal transport, NeurIPS 33
  (2020) 17559--17570.

\bibitem{kingma2014adam}
D.~P. Kingma, J.~Ba, Adam: A method for stochastic optimization, in: ICLR,
  2015.

\bibitem{caputo2014imageclef}
B.~Caputo, H.~M{\"u}ller, J.~Martinez-Gomez, M.~Villegas, B.~Acar, N.~Patricia,
  N.~Marvasti, S.~{\"U}sk{\"u}darl{\i}, R.~Paredes, M.~Cazorla, et~al.,
  Imageclef 2014: Overview and analysis of the results, in: International
  Conference of the Cross-Language Evaluation Forum for European Languages,
  2014, pp. 192--211.

\bibitem{saenko2010adapting}
K.~Saenko, B.~Kulis, M.~Fritz, T.~Darrell, Adapting visual category models to
  new domains, in: ECCV, 2010, pp. 213--226.

\bibitem{peng2017visda}
X.~Peng, B.~Usman, N.~Kaushik, J.~Hoffman, D.~Wang, K.~Saenko, Visda: The
  visual domain adaptation challenge, arXiv preprint arXiv:1710.06924.

\bibitem{venkateswara2017deep}
H.~Venkateswara, J.~Eusebio, S.~Chakraborty, S.~Panchanathan, Deep hashing
  network for unsupervised domain adaptation, in: CVPR, 2017, pp. 5018--5027.

\bibitem{he2016deep}
K.~He, X.~Zhang, S.~Ren, J.~Sun, Deep residual learning for image recognition,
  in: CVPR, 2016, pp. 770--778.

\bibitem{xu2019larger}
R.~Xu, G.~Li, J.~Yang, L.~Lin, Larger norm more transferable: An adaptive
  feature norm approach for unsupervised domain adaptation, in: CVPR, 2019, pp.
  1426--1435.

\bibitem{ma2024small}
Y.~Ma, X.~Yao, R.~Chen, R.~Li, X.~Shen, B.~Yu, Small is beautiful: Compressing
  deep neural networks for partial domain adaptation, IEEE TNNLS 35~(3) (2024)
  3575--3585.

\bibitem{kim2022adaptive}
Y.~Kim, S.~Hong, Adaptive graph adversarial networks for partial domain
  adaptation, IEEE TCSVT 32~(1) (2022) 172--182.

\bibitem{li2023critical}
S.~Li, K.~Gong, B.~Xie, C.~H. Liu, W.~Cao, S.~Tian, Critical classes and
  samples discovering for partial domain adaptation, IEEE TCYB 53~(9) (2023)
  5641--5654.

\bibitem{li2023partial}
W.~Li, S.~Chen, Partial domain adaptation without domain alignment, IEEE TPAMI
  45~(7) (2023) 8787--8797.

\bibitem{sahoo2023select}
A.~Sahoo, R.~Panda, R.~Feris, K.~Saenko, A.~Das, Select, label, and mix:
  Learning discriminative invariant feature representations for partial domain
  adaptation, in: CVPR, 2023, pp. 4210--4219.

\bibitem{van2008visualizing}
L.~Van~der Maaten, G.~Hinton, Visualizing data using t-sne, JMLR 9~(11).

\end{thebibliography}

\end{document}